\documentclass[10pt,journal,compsoc]{IEEEtran}
\usepackage{times}
\usepackage{helvet}
\usepackage{courier}
\usepackage{amssymb}
\usepackage{amsmath}
\usepackage{graphicx}
\usepackage{subfigure}
\usepackage{amsthm}
\usepackage{multirow}

\newcommand{\beq}{\begin{equation}}
\newcommand{\eeq}{\end{equation}}
\newcommand{\bea}{\end{align}}
\newcommand{\eea}{\end{align}}
\newcommand{\beas}{\end{align*}}
\newcommand{\eeas}{\end{align*}}
\newcommand{\pt}{{\mathbf{P}_{t}}}
\newcommand{\pstat}{{\mathbf{P}^*}}
\newcommand{\at}{{\mathbf{A}_{t}}}
\newcommand{\atl}{{\mathbf{A}_{t-1}}}
\newcommand{\oml}{{{\Omega}_{t-1}}}
\newcommand{\gal}{{\Gamma_{t-1}}}

\newcommand{\cxxt}{\mathbf{C}_{t,\mathbf{XX}}}

\newcommand{\E}{\mathbb{E}}

\newcommand{\by}{\mathbf{y}}
\newcommand{\bx}{\mathbf{x}}

\newtheorem{theorem}{Theorem}[section]

\newtheorem{lemma}[theorem]{Lemma}
\newtheorem{remark}[theorem]{Remark}

\newtheorem{property}[theorem]{Property}

\ifCLASSOPTIONcompsoc
  \usepackage[nocompress]{cite}
\else
  \usepackage{cite}
\fi
\ifCLASSINFOpdf
\else
\fi
\begin{document}

\title{Dynamic Structure Embedded Online Multiple-Output Regression for Stream Data}

\author{Changsheng~Li,
        Fan~Wei$^{\ast}$,
        Weishan~Dong$^{\ast}$,
        Qingshan~Liu,~\IEEEmembership{Senior Member,~IEEE,}
        Xiangfeng~Wang,
        and~Xin~Zhang
\IEEEcompsocitemizethanks{\IEEEcompsocthanksitem $^\ast$ Authors contributed equally.
\vspace{-1em}}

\IEEEcompsocitemizethanks{\IEEEcompsocthanksitem Corresponding author: Q. Liu.
\vspace{-1em}}
\IEEEcompsocitemizethanks{\IEEEcompsocthanksitem C. Li, W. Dong, and X. Zhang are with IBM Research-China, Beijing, China.
E-mail: \{lcsheng, dongweis, zxin\}$@$cn.ibm.com.
\vspace{-1em}}
\IEEEcompsocitemizethanks{\IEEEcompsocthanksitem F. Wei is with Department of Mathematics, Standford University, USA. E-mail: fanwei@math.stanford.edu.
\vspace{-1em}}
\IEEEcompsocitemizethanks{\IEEEcompsocthanksitem Q. Liu is with the B-DAT Laboratory at the school of Information and Control, Nanjing University of Information Science and Technology, China. E-mail:  qsliu@nuist.edu.cn.
\vspace{-1em}}
\IEEEcompsocitemizethanks{\IEEEcompsocthanksitem X. Wang is with Software Engineering Institute, East China Normal University, China. E-mail: xfwang@sei.ecnu.edu.cn.}

}


\IEEEtitleabstractindextext{%
\begin{abstract}
Online multiple-output regression is an important machine learning technique for modeling, predicting, and compressing multi-dimensional correlated data streams. In this paper, we propose a novel online multiple-output regression method, called MORES, for stream data.
MORES can \emph{dynamically} learn the structure of the coefficients change in each update step to facilitate the model's continuous refinement.
We observe that limited expressive ability of the regression model, especially in the preliminary stage of online update, often leads to the variables in the residual errors being dependent. In light of this point, MORES intends to \emph{dynamically} learn and leverage the structure of the residual errors to improve the prediction accuracy.
Moreover, we define three statistical variables to \emph{exactly} represent all the seen samples for \emph{incrementally} calculating prediction loss in each online update round, which can avoid loading all the training data into memory for updating model, and also effectively prevent drastic fluctuation of the model in the presence of noise. Furthermore, we introduce a forgetting factor to set different weights on samples so as to track the data streams' evolving characteristics quickly from the latest samples.
Experiments on one synthetic dataset and three real-world datasets validate the effectiveness of the proposed method. In addition, the update speed of MORES is at least 2000 samples processed per second on the three real-world datasets, more than 15 times faster than the state-of-the-art online learning algorithm.
\end{abstract}

}

\maketitle

\IEEEdisplaynontitleabstractindextext

%
\IEEEpeerreviewmaketitle
\section{Introduction}
Data streams arise in many scenarios, such as online transactions in the financial market, Internet traffic and so on~\cite{Lamport:hanjiawei}. Unlike traditional datasets in batch mode, a data stream should be viewed as a potentially infinite process collecting data with varying update rates, as well as continuously evolving over time.
In the context of data streams, although many research issues, such as classification \cite{domingos2000mining,aggarwal2006framework,muhlbaier2009learn}, clustering \cite{aggarwal2003framework,chen2007density,wang2013svstream}, active learning \cite{zhu2007active,chu2011unbiased,zliobaite2014active}, online feature selection \cite{wu2013online,jialei}, multi-task learning \cite{saha2011online,dekel2007online}, change point detection \cite{ho2010martingale,yamada2013change}, etc., have been extensively studied over the last decade, little attention is paid to multiple-output regression. However, multiple-output regression also has a great variety
of potential applications on data streams, including weather forecast, air quality prediction, etc.

In batch data processing, the purpose of multiple-output regression is to learn a mapping $\Phi$ from an input space $\mathbb{R}^d$ to an output space $\mathbb{R}^m$ on the whole training dataset. Based on the learned $\Phi$, we can simultaneously predict multiple output variables $\mathbf{y}\in \mathbb{R}^m$ to a given new input vector $\mathbf{x}\in \mathbb{R}^d$ by: $\mathbf{y}=\Phi(\mathbf{x})$.
According to the property of the learned mapping $\Phi$, multiple-output regression techniques can be divided into linear and nonlinear ones. Since linear methods usually have low complexities and reliable prediction performance, they have attracted more attention in the past.
Lots of batch multiple-output regression algorithms have been proposed
\cite{kim2009multivariate,sanchez2004svm,rothman2010sparse,sohn2012joint,rai2012simultaneously,kim2012tree,gonen2014kernelized,liu2014multivariate}.
However, in streaming environments, the data are not stored in a batch mode, but it arrives sequentially and continuously. If using these batch methods to re-train the models for streaming data, the computational complexity and memory complexity will increase sharply. Moreover, when the size of the arrived data becomes large, it is also impractical to load all the historical data into memory for model training.
Therefore, it is necessary to develop an online regression algorithm for simultaneously predicting multiple outputs.

So far, many online regression algorithms for predicting single output variable have been proposed \cite{ma2003accurate,crammer2006online,montana2008learning}. The representative method is online passive-aggressive (PA) algorithm \cite{crammer2006online}. PA is a margin based online learning algorithm, and it has an analytical solution to update model parameters as the new data sample(s) arrives.
Since there are often correlations among outputs, mining the correlation relationships can improve the prediction accuracy of the model \cite{rothman2010sparse}. However, PA only treats each of multiple outputs as an independent task, and thus can not capture any correlations among outputs.
Recently, McWilliams and Montana take advantage of partial least squares (PLS) to build a recursive regression model for online predicting multiple outputs, called iS-PLS \cite{mcwilliams2010sparse}. iS-PLS aims at finding a low-dimensional subspace to make the correlation between inputs and outputs maximized. iS-PLS focuses on the correlation between inputs and outputs, while it does not obviously consider the correlations among outputs.

In addition, many online multi-task learning algorithms can be potentially used for solving online multiple-output regression problem.  The goal of online multi-task learning is to jointly learn the related tasks (typically, classification or regression tasks) in an online fashion, so as to improve generalization across all tasks \cite{dekel2006online,eaton2013ella,ammar2014online}.
Dekel et al. \cite{dekel2006online} propose an online multi-task learning algorithm to capture the relationship between the tasks. They associate each individual prediction with its own individual loss, and then use a global loss function to evaluate the quality of the multiple predictions made on each round. Cavallanti et al. \cite{cavallanti2010linear} explicitly encode the task relationships in a matrix that is assumed to be known beforehand.
However, such an a priori assumption can often be restrictive in practice.
In order to conquer this limitation, Saha et al. \cite{saha2011online} propose to simultaneously learn the task weight vectors, as well their adaptive relationships from the data.
Ruvolo and Eaton \cite{eaton2013ella} aim to maintain a basis to transfers knowledge for learning each new task, and dynamically update the basis to improve previously learned models.
Furthermore, Ruvolo and Eaton \cite{ruvolo2014online} utilize K-SVD to learn multiple tasks, which has lower computational complexity over \cite{eaton2013ella}.
However, in \cite{dekel2006online}, \cite{cavallanti2010linear}, and \cite{saha2011online}, they only use the current new data to update the models on each round, but discard all the historical data, which loses a large amount of useful information and makes the models be easily affected by noise. In \cite{eaton2013ella} and \cite{ruvolo2014online}, they load all the data into the memory and visit a sample multiple times to update the model on each round, which is impractical in the streaming environment.
In addition, because of limited expressive ability of the learned models, especially in the preliminary stage of online learning, there are often correlations among the prediction errors. But the methods above do not consider how to utilize such relationships to further improve the performance of the models.

In this paper, we propose a novel \underline{M}ultiple-\underline{O}utput \underline{RE}gression method for \underline{S}tream data, named as MORES.
MORES works in an incremental fashion.
It aims at \emph{dynamically} learning the structures of both the regression coefficients change and the residual errors to continuously refine the model.
These structures can be learned from three statistical variables which can represent the necessity of the data without information loss.
Specifically, when a new training sample arrives, we transform the update of the regression coefficients into an optimization problem.
In the objective function, we highlight the following three aspects:

First, we take advantage of the matrix \emph{Mahalanobis norm} to measure the divergence between updated regression coefficient matrix and
current regression coefficient matrix, which can learn the structure of the coefficients change in each update step to facilitate model's refinement.
Second,
we utilize the {\emph{Mahalanobis}} distance instead of the Euclidean distance to measure the prediction error, so as to learn the structure of the residual errors and embed such structure to update the model.
Third, we define three statistical variables to store necessary information of all the seen data for exactly and incrementally measuring the prediction error, such that MORES can avoid loading all the data into memory and visiting data multiple times. This also effectively prevents the updated regression coefficient matrix gradually deviating from the true coefficient matrix due to noise and outliers in the data streams. Meanwhile, we introduce a forgetting factor to reweight samples for adapting to data streams' evolvement.
Extensive experiments are conducted on one synthetic dataset and three real-world datasets, and the experimental results demonstrate the effectiveness and efficiency of MORES.

\section{Online Multiple-Output Regression for Streaming Data}
Following the general setting of online learning \cite{crammer2006online}, we assume that the learner first observes an instance $\mathbf{x}_t\in\mathbb{R}^d$ on the $t$-th round, and it simultaneously predicts multiple outputs $\widehat{\mathbf{y}_t} \in \mathbb{R}^m$ based on the current model $\mathbf{P}_{t-1} \in\mathbb{R}^{m\times d}$. After that, the learner receives the true responses $\mathbf{y}_t \in \mathbb{R}^m$ for this instance. Finally, the learner updates the current model $\mathbf{P}_{t-1}$ based on the new data point $(\mathbf{x}_t, \mathbf{y}_t)$.
In this paper, our goal is to \emph{online} update $\mathbf{P}_{t-1}$, such that the updated $\mathbf{P}_{t}$ can predict the outputs for the incoming instance $\mathbf{x}_{t+1}$ as accurately as possible.
The prediction can be expressed in the following form:
\begin{equation}\label{sdform}
\mathbf{y}_{t+1}=\mathbf{P}_t\mathbf{x}_{t+1}+\epsilon_{t+1},
\end{equation}
where $\mathbf{P}_t=[\mathbf{p}_{t,1},\ldots,\mathbf{p}_{t,m}]^T$ denotes the learned regression coefficient matrix on the $t$-th round, and
$\mathbf{p}_{t,i}$ is the regression coefficient vector of the $i$-th output. $\epsilon_{t+1}=[\epsilon_{t+1,1},\ldots,\epsilon_{t+1,m}]^T$ is a vector
consisting of $m$ residual errors.

\subsection{Objective Function}\label{simpleformulation}
In order to obtain $\mathbf{P}_t$ on the $t$-th round, we first propose a simple formulation as:
\begin{equation}\label{somor}
\begin{array}{l}
\mathbf{P}_t=\mathop{\arg\min}\limits_{\mathbf{P}\in\mathbb{R}^{m\times d}}\|\mathbf{P}-\mathbf{P}_{t-1}\|_{\mathbf{F}}^2 \\
s.t. \ \ \|\mathbf{y}_t-\mathbf{P}\mathbf{x}_t\|_2^2\leq \xi, \qquad
\end{array}
\end{equation}
where $\xi$ is a positive parameter that controls the sensitivity to the prediction error.
$\|\cdot\|_{\mathbf{F}}$ denotes the matrix \emph{Frobenius norm}, and $\|\cdot\|_2$ denotes the $l_2$ \emph{norm} of a vector.

The core idea of objective function (\ref{somor}) is as follows: On one hand, it intends to minimize the distance between $\mathbf{P}_t$ and $\mathbf{P}_{t-1}$ to make $\mathbf{P}_t$ close to $\mathbf{P}_{t-1}$ as much as possible, which can retain the information learned on previous rounds.
On the other hand, it requires $\mathbf{P}_t$ to meet the condition: The total prediction error on the current data point $(\mathbf{x}_t,\mathbf{y}_t)$ is less than or equal to $\xi$.

Following \cite{crammer2006online}, the optimization problem defined by (\ref{somor}) can be easily solved by the Lagrange multiplier method.

Although (\ref{somor}) has some merits for online regression prediction of streaming data, it still has the following limitations:

(i) Because of
$\|\mathbf{P}-\mathbf{P}_{t-1}\|_{\mathbf{F}}^2=tr((\mathbf{P}-\mathbf{P}_{t-1})(\mathbf{P}-\mathbf{P}_{t-1})^T)=\sum_{i=1}^m((\mathbf{p}_i-\mathbf{p}_{t-1,i})^T(\mathbf{p}_i-\mathbf{p}_{t-1,i}))$, we can see that the objective function in (\ref{somor}) treats the update of regression
coefficients as $m$ independent tasks. However, in streaming environments, outputs are often dependent, i.e., there are some positive or negative correlations among outputs. For example, correlated outputs could be the concentration of CO$_2$ and fine particles PM$_{2.5}$ in a air quality forecast. Thus updating regression coefficient vectors for all the outputs should not be regarded as completely independent tasks.

(ii) In order to acquire the updater $\mathbf{P}_t$, (\ref{somor}) imposes a constraint on $\mathbf{P}$, i.e.,
$\|\mathbf{y}_t-\mathbf{P}\mathbf{x}_t\|_2^2\leq \xi$. This constraint just refers to the current data point $(\mathbf{x}_t,\mathbf{y}_t)$ but
ignores the historical data points $\mathcal{S}_{t-1}=\{(\mathbf{x}_i,\mathbf{y}_i)\}_{i=1}^{t-1}$. This may lead to the updated coefficient matrix gradually deviating from the true coefficient matrix because of noise and outliers in many practical streaming data applications.

(iii) The constraint of (\ref{somor}) takes advantage of the $l_2$ norm to measure the total prediction error. As we know, the $l_2$ {norm} of a vector assumes that all the variables in the vector are
independent. However, due to limited expressive power of $\mathbf{P}_t$, especially when the round $t$ is small, there are often
correlations between the residual errors. Therefore, the $l_2$ {norm} is often the suboptimal choice for measuring the total prediction
error.

In light of above three limitations, we propose to minimize the following objective function, in order to update the model on round $t$:
\begin{align}\label{obj}
(\mathbf{P}_t,\Omega_t,\Gamma_t)=&\arg\min_{\mathbf{P},\Omega,\Gamma} \ \ J(\mathbf{P},\Omega,\Gamma; \mathbf{P}_{t-1}, \Omega_{t-1}, \mathcal{S}_t) \nonumber\\
=&\arg\min_{\mathbf{P},\Omega,\Gamma}\|\mathbf{P}-\mathbf{P}_{t-1}\|_{\Omega}^2+\alpha\ell(\mathbf{P},\Gamma;\mathcal{S}_t)\nonumber\\
&+\beta\Delta_\phi(\Omega,\Omega_{t-1})+\rho\Delta_\phi(\Omega,\mathbf{I}) +\eta\Delta_\phi(\Gamma,\mathbf{I}) \nonumber\\
s.t. \ \ \ \ \Omega &\succeq 0, \ \ \Gamma\succeq 0, \qquad \qquad \qquad \qquad\qquad \qquad
\end{align}
where $\alpha\geq0$, $\beta\geq0$, $\rho\geq0$, and $\eta\geq0$ are four trade-off parameters. $\|\cdot\|_{\Omega}$ denotes the
matrix \emph{Mahalanobis norm}. $\ell(\mathbf{P}, \Gamma;\mathcal{S}_t)$ is the total prediction error on the data points $\mathcal{S}_t=\{(\mathbf{x}_i,\mathbf{y}_i)\}_{i=1}^{t}$.
$\Delta_\phi(\cdot,\cdot)$ denotes the Bregman divergence \cite{kivinen1997exponentiated} that measures the distance between two matrices. $\Omega\succeq 0$ and $\Gamma\succeq 0$ represent that they are positive semi-definite.

In the objective function (\ref{obj}), the first term aims to learn the structure $\Omega_t$ of the coefficients change from current matrix $\mathbf{P}_{t-1}$ to updated matrix $\mathbf{P}_t$, and leverage $\Omega_t$ to measure the divergence between $\mathbf{P}_t$ and $\mathbf{P}_{t-1}$ on round $t$.
The second term intends to mine the underlying structure $\Gamma_t$ existing in the residual errors, and take advantage of $\Gamma_t$ to measure the total prediction error on both the current data and the historical data in the $t$-th update. Here we define three statistical variables to store necessary information of data for lowering memory complexity, and introduce a forgetting factor to adapt to the evolving data streams (See (\ref{loss_fuction}), (\ref{new_loss}) for details).
The third term aims at keeping $\Omega_t$ updated in a conservative strategy to reduce the influence of noise. The last two terms are regularization terms to prevent $\Omega_t$ and $\Gamma_t$ deviating from the identity matrix $\mathbf{I}$ too much. In the meantime, they also help ensuring that $\Omega_t$ and $\Gamma_t$ are away from 0 otherwise it is meaningless to minimize $J$. We will see in Section \ref{optimizationpro} that the eigenvalues of $\Omega_t$ and $\Gamma_t$ are always bounded between 0 and 1 if initialized with eigenvalues between 0 and 1, and the eigenvalues are monotone increasing after each step of iteration, making the iteration procedure well-controlled.
Next, we will respectively explain the first three terms in detail.

\noindent\textbf{The first term}: In order to capture the structure of the regression coefficient change in each update step, we introduce \emph{Mahalanobis norm} of the matrix $(\mathbf{P}-\mathbf{P}_{t-1})$ to measure the divergence between $\mathbf{P}$ and $\mathbf{P}_{t-1}$. The \emph{Mahalanobis norm} is expressed as:
\begin{align}\label{mahanorm}
\|\mathbf{P}-\mathbf{P}_{t-1}\|_{\Omega}=\sqrt{tr((\mathbf{P}-\mathbf{P}_{t-1})^T\Omega(\mathbf{P}-\mathbf{P}_{t-1}))} \nonumber\\
=\sqrt{\sum\nolimits_{i=1}^d(\mathbf{P}(i)-\mathbf{P}_{t-1}(i))^T\Omega(\mathbf{P}(i)-\mathbf{P}_{t-1}(i))},
\end{align}
where $\mathbf{P}(i)$ denotes the $i$-th column of $\mathbf{P}$. When $\Omega$ is set to the identity matrix, the
\emph{Mahalanobis norm} of the matrix is reduced to the \emph{Frobenius norm}.
In (\ref{mahanorm}), the term $(\mathbf{P}(i)-\mathbf{P}_{t-1}(i))^T\Omega(\mathbf{P}(i)-\mathbf{P}_{t-1}(i))$ is actually the \emph{Mahalanobis} distance between $\mathbf{P}(i)$ and $\mathbf{P}_{t-1}(i)$, where $\Omega$ encodes the correlations between the variables of the $i$-th column of the regression coefficient matrix on round $t$ \cite{zhang2012convex}.
Therefore, $\|\mathbf{P}-\mathbf{P}_{t-1}\|_{\Omega}^2$ can be viewed as a summation of $d$ {Mahalanobis} distances, each of which measures the distance between the corresponding column vectors of $\mathbf{P}$ and $\mathbf{P}_{t-1}$.

\noindent\textbf{The second term}: The loss function $\ell(\mathbf{P}, \Gamma;\mathcal{S})$ measures the prediction error on $\mathcal{S}$, which is defined as:
\begin{align}\label{loss_fuction}
\ell(\mathbf{P},\Gamma;\mathcal{S})=\sum\nolimits_{i=1}^t\mu^{t-i}(\mathbf{y}_i-\mathbf{P}\mathbf{x}_i)^T\Gamma(\mathbf{y}_i-\mathbf{P}\mathbf{x}_i),
\end{align}
where $0\leq\mu\leq1$ is a forgetting factor\footnote{When $\mu=0$, and $t-i=0$, we define $\mu^{t-i}=1$ to ensure consistency.}. When $\mu=0$, the prediction loss is only measured on the current sample without referring to any historical samples.
When $\mu=1$, all the samples have equal weight to contribute to the prediction loss.
When $0<\mu<1$, all the samples have different contributions to the prediction loss.
As a matter of fact, the function of $\mu$ is similar to a new form of time window on samples.
The farther the historical sample is from the current sample in the time domain, the lower its importance is, which will fit in the evolving characteristic of data streams well.
The matrix $\Gamma_t$ embeds the correlation relationships among the residual errors on $t$-th update.
The term $(\mathbf{y}_i-\mathbf{P}\mathbf{x}_i)^T\Gamma(\mathbf{y}_i-\mathbf{P}\mathbf{x}_i)$ measures the {Mahalnobis} distance between the true value $\mathbf{y}_i$ and the predicted value $\mathbf{P}\mathbf{x}_i$, which can remove the influence of the residual errors' correlations on distance calculation \cite{weinberger2005distance}.

In streaming environments, it is impractical to load all the historical data into memory or scan a sample multiple times. An effective way to handle this issue is to define some statistical variables to store necessary information of the samples. In this paper, we introduce three statistical variables to realize lossless compression of the data. To do this, we will make use of the following property and lemma.

\begin{property}
Given a set of arbitrary sequence vectors, $\mathbf{x}_1,\!\ldots\!,\mathbf{x}_t$, and a constant $\mu$, if $\mathbf{C}_t\!=\!\sum\nolimits_{i=1}^t\!\mu^{t\!-\!i}\mathbf{x}_i\mathbf{x}_i^T$, then $\mathbf{C}_t\!=\!\mu \mathbf{C}_{t\!-\!1}\!+\!\mathbf{x}_t\mathbf{x}_t^T$, where $t$ is the timestamp.
\end{property}

\begin{lemma}
The loss function (\ref{loss_fuction}) can be expressed as
\begin{align}\label{new_loss}
\ell=tr(\Gamma\mathbf{C}_{t,\mathbf{YY}})+tr(\mathbf{P}^T\Gamma\mathbf{P}\mathbf{C}_{t,\mathbf{XX}})-2tr(\Gamma\mathbf{P}\mathbf{C}_{t,\mathbf{XY}}).
\end{align}
where $\mathbf{C}_{t,\mathbf{YY}}$, $\mathbf{C}_{t,\mathbf{XY}}$, and $\mathbf{C}_{t,\mathbf{XX}}$ are three variables, which are respectively defined as:
\begin{align}
\label{sv1}
\mathbf{C}_{t,\mathbf{YY}}&=\mu\mathbf{C}_{t-1,\mathbf{YY}}+\mathbf{y}_t\mathbf{y}_t^T. \\
\label{sv2}
\mathbf{C}_{t,\mathbf{XY}}&=\mu\mathbf{C}_{t-1,\mathbf{XY}}+\mathbf{x}_t\mathbf{y}_t^T.\\
\label{sv3}
\mathbf{C}_{t,\mathbf{XX}}&=\mu\mathbf{C}_{t-1,\mathbf{XX}}+\mathbf{x}_t\mathbf{x}_t^T.
\end{align}
\end{lemma}

\begin{proof}
Based on (\ref{loss_fuction}), we have
\begin{align} \label{proof1}
\ell (&\mathbf{P}, \Gamma; \mathcal{S})=\sum\nolimits_{i=1}^t\mu^{t-i}(\mathbf{y}_i-\mathbf{P}\mathbf{x}_i)^T\Gamma(\mathbf{y}_i-\mathbf{P}\mathbf{x}_i)\nonumber\\
=&\sum\nolimits_{i=1}^t\mu^{t-i}(tr(\mathbf{y}_i^T\Gamma\mathbf{y}_i)+tr(\mathbf{x}^T_i\mathbf{P}^T\Gamma\mathbf{P}\mathbf{x}_i))\nonumber\\
&-2\sum\nolimits_{i=1}^t\mu^{t-i}tr(\mathbf{y}_i^T\Gamma\mathbf{P}\mathbf{x}_i)\nonumber\\
=&\sum\nolimits_{i=1}^t\mu^{t-i}(tr(\Gamma\mathbf{y}_i\mathbf{y}_i^T)+tr(\mathbf{P}^T\Gamma\mathbf{P}\mathbf{x}_i\mathbf{x}^T_i))\nonumber\\
&-2\sum\nolimits_{i=1}^t\mu^{t-i}tr(\Gamma\mathbf{P}\mathbf{x}_i\mathbf{y}_i^T)\nonumber\\
=&tr(\Gamma \sum\nolimits_{i=1}^t\mu^{t-i}\mathbf{y}_i\mathbf{y}_i^T)+tr(\mathbf{P}^T\Gamma\mathbf{P}\sum\nolimits_{i=1}^t\mu^{t-i}\mathbf{x}_i\mathbf{x}^T_i)\nonumber\\
&-2tr(\Gamma\mathbf{P}\sum\nolimits_{i=1}^t\mu^{t-i}\mathbf{x}_i\mathbf{y}_i^T).
\end{align}

Defining the matrix variables $\mathbf{C}_{t,\mathbf{YY}}, \mathbf{C}_{t,\mathbf{XY}}$, and $\mathbf{C}_{t,\mathbf{XX}}$ as:
\begin{align}\label{ctyy}
\mathbf{C}_{t,\mathbf{YY}}=\sum\nolimits_{i=1}^t\mu^{t-i}\mathbf{y}_i\mathbf{y}_i^T \\
\label{ctxy}
\mathbf{C}_{t,\mathbf{XY}}=\sum\nolimits_{i=1}^t\mu^{t-i}\mathbf{x}_i\mathbf{y}_i^T \\
\label{ctxx}
\mathbf{C}_{t,\mathbf{XX}}=\sum\nolimits_{i=1}^t\mu^{t-i}\mathbf{x}_i\mathbf{x}^T_i
\end{align}

Substituting $\mathbf{C}_{t,\mathbf{YY}}, \mathbf{C}_{t,\mathbf{XY}},$ and $\mathbf{C}_{t,\mathbf{XX}}$ into (\ref{proof1}), the loss function (\ref{loss_fuction}) becomes (\ref{new_loss}). Based on the {Property 1}, (\ref{ctyy}), (\ref{ctxy}), and (\ref{ctxx}) can be expressed by (\ref{sv1}), (\ref{sv2}), and (\ref{sv3}), respectively.
\end{proof}
%


When a new data point $(\mathbf{x}_{t+1},\mathbf{y}_{t+1})$ arrives on round $t\!+\!1$, we first update the statistical variables
$\mathbf{C}_{t+1,\mathbf{YY}}$, $\mathbf{C}_{t+1,\mathbf{XY}}$, $\mathbf{C}_{t+1,\mathbf{XX}}$ based on (\ref{sv1}), (\ref{sv2}), and (\ref{sv3}) respectively, whose memory complexity is a constant: $O(m^2+md+d^2)$.
After that, the prediction loss $\ell$ can be measured by (\ref{new_loss}), so it is no longer necessary to load all the training samples into memory or visit a training sample multiple times.

In summary, the loss function (\ref{loss_fuction}) has the following merits: First, it can dynamically learn the structure of the residual errors as the samples continuously arrive, and embed the structure information to effectively measure the true distance between the predicted value and the ground truth.
Second, the loss is measured based on all the seen samples not just the current sample, which can cut down on the influence of noise on model's update.
Third, on each round, instead of loading all the samples into memory, the loss can be measured just relying on three defined statistical variables without information loss, as expressed in (\ref{new_loss}).
Furthermore, by introducing the factor $\mu$ to weight the samples, MORES can fit in streaming data's evolvement well.

\noindent\textbf{{The third term}}: In order to restrain $\phi$ fluctuating drastically on each round, we hope that the divergence $\Delta_\phi(\Omega,\Omega_{t-1})$ is as small as possible. $\Delta_\phi(\Omega,\Omega_{t-1})$ is defined as:
\begin{align}
\Delta_\phi(\Omega,\Omega_{t-1})=\phi({\Omega})-\phi({\Omega_{t-1}})-tr(g({\Omega_{t-1}})(\Omega-\Omega_{t-1})),\nonumber
\end{align}
where $\phi$ is a real-valued strictly convex differentiable function on the parameter domain $\mathbb{R}^{m\times m}$.
$g({\Omega_{t-1}})=\nabla_\Omega\phi({\Omega})|_{\Omega_{t-1}}$. In this paper, we employ the \emph{LogDet} matrix divergence metric to measure the distance between two matrices, because of its good properties stated in \cite{tsuda2005matrix}. \emph{LogDet} can be expressed as
\begin{align}
\Delta_\phi(\Omega,\Omega_{t-1})&={log} \frac{{det}(\Omega_{t-1})}{{det}(\Omega)}+tr(\Omega_{t-1}^{-1}\Omega)-m,\nonumber
\end{align}
where $\phi({\Omega})=-{{log}} ({det}(\Omega))$, and $det(\cdot)$ denotes the determinant of a matrix.
\subsection{Optimization Procedure}\label{optimizationpro}
The objective function (\ref{obj}) is not convex with respect to all variables, but it is convex with each variable when others are fixed.
We adopt an alternating optimization strategy to solve (\ref{obj}), which can find local minima.

\textbf{Optimizing} $\mathbf{P}_t$ \textbf{given} $\Omega_{t-1}$ \textbf{and} $\Gamma_{t-1}$: When $\Omega_{t-1}$ {and} $\Gamma_{t-1}$ are fixed, (\ref{obj}) is then unconstrained and convex.
$\mathbf{P}_t$ can be obtained by minimizing the following objective function:
\begin{align}\label{objfun1}
\mathbf{P}_t=&\arg\min_{\mathbf{P}} J_1(\mathbf{P}; \mathbf{P}_{t-1}, \Omega_{t-1}, \Gamma_{t-1}, \mathcal{S}_t) \nonumber\\
=&\arg\min_{\mathbf{P}}\|\mathbf{P}-\mathbf{P}_{t-1}\|_{\Omega_{t-1}}^2+\alpha\ell(\mathbf{P},\Gamma_{t-1};\mathcal{S}_t).\nonumber
\end{align}

The necessary condition for the optimality is:
\beq
\frac{\partial J_1(\mathbf{\footnotesize{P}}; \mathbf{P}_{t-1}, \Omega_{t-1}, \Gamma_{t-1}, \mathcal{S}_t)}{\partial\mathbf{P}}=0. \label{aaaa}
\eeq
This implies
 \beq
 \Omega_{t-1}\mathbf{P}+\alpha\Gamma_{t-1}\mathbf{P} \mathbf{C}_{t,\mathbf{XX}} =\Omega_{t-1} {\mathbf{P}}_{t-1}+\alpha \Gamma_{t-1}\mathbf{C}_{t,\mathbf{XY}}^T. \label{lyap}
\eeq
To solve (\ref{lyap}), we will use the property that $\oml, \gal$ are always positive definite as we will show later, and thus we can write (\ref{lyap}) as
\[ \Gamma_{t-1}^{-1} \Omega_{t-1} \mathbf{P} + \alpha \mathbf{P} \cxxt  = \Gamma_{t-1}^{-1} \Omega_{t-1} \mathbf{P}_{t-1}  + \alpha \mathbf{C}_{t,\mathbf{XY}}^T = \mathbf{Z}. \]

Since $\cxxt$ is positive semi-definite and symmetric, we can do eigenvalue decomposition with all non-negative eigenvalues, obtaining $\cxxt = \mathbf{V} \Theta \mathbf{V}^{-1}$.
For $\Gamma_{t-1}^{-1} \Omega_{t-1}$, we can do Schur decomposition, obtaining $\Gamma_{t-1}^{-1} \Omega_{t-1} =  \mathbf{U L U}^{-1}$ where $\mathbf{L}$ is an upper-triangular matrix. Then (\ref{lyap}) can be written as
\[ \mathbf{UL U}^{-1} \mathbf{P} + \mathbf{P} \alpha \mathbf{V} \Theta \mathbf{V}^{-1} = \mathbf{Z}. \]
Therefore
\[ \mathbf{L U}^{-1} \mathbf{P} \mathbf{V} + \mathbf{U}^{-1} \mathbf{P} \mathbf{V} (\alpha \Theta) = \mathbf{U}^{-1} \mathbf{Z V}. \]

Let $\mathbf{U}^{-1} \mathbf{P} \mathbf{V}  = \tilde{\mathbf{P}}$, we have an equation in a simpler form, i.e.,
 \[ \mathbf{L} \tilde{\mathbf{P}} +  \tilde{\mathbf{P}} (\alpha \Theta) = \mathbf{U}^{-1} \mathbf{Z V}. \]
 By writing out the the product of matrices on the left-hand side entry-wise, we can easily obtaining the values of $\tilde{\mathbf{P}}$ starting from the last row. Specifically,
 if $\tilde{\mathbf{P}}$ is $m \times d$, then all the elements of the last row of $\mathbf{P}$ has the explicit form
 \begin{align}\label{phat1}
 \tilde{\mathbf{P}}_{m, i} = \frac{(\mathbf{U}^{-1} \mathbf{Z V})_{m,i}}{\mathbf{L}_{m,m} + \alpha \Theta_{i,i}}, i=1,2,\ldots, d.
 \end{align}
 And we can obtain all the elements of the second last row of $\mathbf{P}$ as
 \begin{align}\label{phat2}
 \tilde{\mathbf{P}}_{m\!-\!1, i} = \frac{(\mathbf{U}^{-1} \mathbf{Z V})_{m-1,i}}{\mathbf{L}_{m-1,m-1} \!+\! \alpha \Theta_{i,i}} \!-\! \frac{\mathbf{L}_{m-1, m}}{\mathbf{L}_{m, m} \!+\! \Theta_{i,i}}\tilde{\mathbf{P}}_{m, i}, i=1,2,\ldots d.
 \end{align}
 Similarly, we can easily obtain the values of the $j$-th row of $\tilde{\mathbf{P}}$ from later rows, i.e., obtaining $\tilde{\mathbf{P}}_{j,i}$ from $\tilde{\mathbf{P}}_{k,l}$ with $k>j$. Therefore by this back propagation idea, we can efficiently obtain the matrix $\tilde{\mathbf{P}}$.
 And then we can obtain $\mathbf{P}_t$  by
 \begin{align}\label{p}
 \mathbf{P}_t = \mathbf{U} \tilde{\mathbf{P}} \mathbf{V}^{-1}.
 \end{align}

 \begin{remark}
 Actually, the procedure above can be much more efficient. This is based on the fact that since $\gal, \oml$ are both positive definite Hermitian matrices, then $\mathbf{\Gamma}^{-1}_{t-1} \Omega_{t-1}$ can have eigen-decomposition. We will prove this fact in Section \ref{theory}. This means that in the Schur decomposition $\Gamma_{t-1}^{-1} \Omega_{t-1} =  \mathbf{U L U}^{-1}$ above, the upper-triangular matrix $\mathbf{L}$ is actually a diagonal matrix. Therefore ${\mathbf{L}}_{i,j} = 0$ for all $i \neq j$.
 Thus in (\ref{phat2}), we can easily have a simpler formula
\begin{align}\label{phat3}
 \tilde{\mathbf{P}}_{j, i}= \frac{{({\mathbf{U}}^{-1} \mathbf{Z V})}_{j,i}}{\mathbf{L}_{j,j} + \alpha \Theta_{i,i}}.
 \end{align}

 \end{remark}

{{\textbf{Optimizing}} $\Omega_t$} \textbf{given} $\mathbf{P}_t$ {\textbf{and}} $\Gamma_{t-1}$: When $\mathbf{P}_t$ {and} $\Gamma_{t-1}$ are fixed, solving $\Omega_t$ becomes a convex optimization problem as:
\begin{align}\label{objfun2}
\Omega_t=&\arg\min_{\Omega\succeq 0} J_2(\Omega; \mathbf{P}_t, \mathbf{P}_{t-1}, \Omega_{t-1}) \nonumber\\
=&\arg\min_{\Omega\succeq 0}\|\mathbf{P}_t-\mathbf{P}_{t-1}\|_{\Omega}^2 +\beta\Delta_\phi(\Omega,\Omega_{t-1}) + \rho \Delta_\phi(\Omega, \mathbf{I}).\nonumber
\end{align}

We want to solve
\beq \frac{\partial J_2(\Omega; \mathbf{P}_t, \mathbf{P}_{t-1}, \Omega_{t-1})}{\partial (\Omega)}=0. \label{po}
\eeq
Notice that $\Omega$ does not have all the entries being independent. Instead, it is a symmetric matrix $\Omega=\Omega^T$. We know for a symmetric matrix,
\[ \dfrac{\partial  tr (\mathbf{Q} \Omega \mathbf{Q}^T)}{\partial \Omega} = 2 \mathbf{QQ}^T - diag(\mathbf{QQ}^T), \]
where $diag(\mathbf{M})$ means a diagonal matrix consisting of only the diagonal entries of $\mathbf{M}$.

We define a linear map $s$ applied to matrix $\mathbf{M} \in \mathbb{R}^{n \times n}$ defined by:
\[s(\mathbf{M}) = 2\mathbf{M} - diag(\mathbf{M}). \]
Then it is easy to see that $ker(s) = 0$, where $ker(s)$ represents the kernel of the specified function, i.e., the subspace of the domain which is mapped to 0 by the linear map $s$. Thus $s$ is a surjective and injective linear map from $\mathbb{R}^{n \times n}$ to $\mathbb{R}^{n \times n}$.

To solve (\ref{po}), we have that
\[ s(\Omega^{-1}) = \frac{\beta( s(\Omega_{t-1}^{-1}) \!+\! \rho  s(\mathbf{I}) \!+\!  s((\mathbf{P}_t\!-\!\mathbf{P}_{t-1})(\mathbf{P}_t\!-\!\mathbf{P}_{t-1})^T))}{\beta + \rho}. \]
Since $s$ is a linear map that is both injective and surjective, we have that $s(A) = s(B)$ if and only if $A = B$. Hence, we have
\begin{align}\label{omegatt}
\Omega_t=\left(\frac{1}{ \beta\!+\! \rho} (\beta \Omega_{t-1}^{-1} \!+\! \rho \mathbf{I} \!+\! (\mathbf{P}_t \!-\!\mathbf{P}_{t-1})(\mathbf{P}_t \!-\!\mathbf{P}_{t-1})^T)\right)^{-1}.
\end{align}
From the right hand side, it is easy to see that $\Omega_t$ is symmetric, thus having all eigenvalues being real.

We want to prove that $\Omega_t$ is positive semi-definite and has eigenvalues no more than 1.
For simplicity, we use the matrix $\mathbf{M}$ to represent $(\mathbf{P}_t-\mathbf{P}_{t-1})(\mathbf{P}_t-\mathbf{P}_{t-1})^T$.

It is quite straightforward that $\Omega_t$ is positive semi-definite.
This is because $\Omega_{t-1}$ is positive semi-definite, then $\Omega_{t-1}^{-1}$ is positive semi-definite.  Together with the fact that $\mathbf{I}$ and $\mathbf{M}$ have positive eigenvalues, we can conclude that $\Omega_t$ is positive definite.

We now show that it has all eigenvalues being no more than 1. We prove it by induction. Suppose that we initialize with an $\Omega_0$ having all eigenvalues no more than 1, and suppose that the conclusion also holds for $\Omega_{t-1}$. Notice that $\Omega_{t-1}^{-1}$ has eigenvalues all no less than 1. $\mathbf{M}$ has eigenvalues being non-negative, and $\mathbf{I}$ has all eigenvalues being 1. Hence $\Omega_t^{-1}$ has eigenvalues no less than $\frac{1}{\beta + \rho} (\beta + \rho + 0) = 1$. Thus $\Omega_t$ has eigenvalues with values no more than 1. We have proved that the $\Omega_t$ are positive semi-definite and having eigenvalues between 0 and 1.

\textbf{Optimizing} $\Gamma_t$ \textbf{given} $\mathbf{P}_t$ \textbf{and} $\Omega_t$: When $\mathbf{P}_t$ {and} $\Omega_t$ are fixed, $\Gamma_t$ can be obtained by solving the following convex optimization problem:
\begin{align}\label{objfun4}
\Gamma_t&=\arg\min_{\Gamma\succeq 0} J_3(\Gamma; \mathbf{P}_t, \mathcal{S}_t)\nonumber\\
&=\arg\min_{\Gamma\succeq 0}\alpha \ell(\Gamma; \mathbf{P}_t, \mathcal{S}_t)+ \eta \Delta_\phi(\Gamma, \mathbf{I}).
\end{align}

The necessary condition for the optimality is $\frac{\partial J_3(\Gamma; \mathbf{P}_t, \mathcal{S}_t)}{\partial \Gamma}=0$. Therefore, we obtain the following closed form solution:
\begin{align}\label{gam}
\Gamma_t \!=\!(\mathbf{I} \!+\! \frac{\eta}{\alpha}(\mathbf{C}_{t,\mathbf{YY}}\!-\!
\mathbf{\small{C}}_{t,\mathbf{XY}}^T\mathbf{P}_t^T
\!-\!
\mathbf{P}_t\mathbf{C}_{t,\mathbf{XY}}
\!+\!
\mathbf{P}_t\mathbf{C}_{t,\mathbf{XX}}\mathbf{P}_t^T))^{-1}.
\end{align}

After obtaining the solution $\Gamma_t$, we want to show that $\Gamma_t$ is positive definite and has eigenvalues no more than 1 if we initialize with a $\Gamma_0$ being positive definite and having eigenvalues no more than 1.

To show that $\Gamma_t$ is positive-definite, it suffices to show that $\mathbf{C}_{t,\mathbf{YY}}\!-\!
\mathbf{\small{C}}_{t,\mathbf{XY}}^T\mathbf{P}_t^T
-
\mathbf{P}_t\mathbf{C}_{t,\mathbf{XY}}
+
\mathbf{P}_t\mathbf{C}_{t,\mathbf{XX}}\mathbf{P}_t^T$ is positive semi-definite, and it suffices to show that for each $i$,
$\by_i \by_i^T - \by_i \bx_i^T \pt^T - \pt^T \bx_i \by_i^T + \pt \bx_i \bx_i^T \pt^T$ is positive semi-definite.

For each pair $\bx_i, \by_i$ we want to find a ${\mathbf{P}_i^*}$ such that $\by_i = {\mathbf{P}_i^*} \bx_i$. This is only impossible if $\bx_i = \mathbf{0}$ but $\by_i \neq \mathbf{0}$. In the later case we are already done since $\by_i \by_i^T - \by_i \bx_i^T \pt^T - \pt^T \bx_i \by_i^T + \pt \bx_i \bx_i^T \pt^T = \by_i \by_i^T$ which is positive semi-definite.

Notice that if the data have no noise and strictly follow a linear condition going through the origin, then ${\mathbf{P}_i^*}$ are all the same, and the ${\mathbf{P}_i^*}$ will be the true $\mathbf{P}$.

Notice that
\begin{align}
&  \by_i \by_i^T - \by_i \bx_i^T \pt^T - \pt^T \bx_i \by_i^T + \pt \bx_i \bx_i^T \pt^T \nonumber\\
 =& {\mathbf{P}_i^*} \bx_i \bx_i^T {\mathbf{P}_i^*} - {\mathbf{P}_i^*} \bx_i \bx_i^T \pt - \pt \bx_i \bx_i^T {\mathbf{P}_i^*}^T + \pt \bx_i \bx_i^T \pt^T  \nonumber\\
 = & {\mathbf{P}_i^*} \bx_i \bx_i^T ({\mathbf{P}_i^*} - \pt^T)  - \pt \bx_i \bx_i^T ({\mathbf{P}_i^*} - \pt) \nonumber\\
 = & ({\mathbf{P}_i^*} -\mathbf{P}_t) \bx_i \bx_i^T ({\mathbf{P}_i^*} -\mathbf{P}_t)^T.\label{compute}
\end{align}
Thus this is positive semi-definite.
We can employ the same strategy to prove that $\Gamma_t$ has eigenvalues between 0 and 1.

Finally, we summarize the procedure of MORES in Alg. 1.
\begin{table}
\begin{center}
\label{mores11}
\begin{tabular}{l}
\hline
\textbf{Alg. 1} \   Multiple-Output REgression for Streaming Data(MORES) \\
\hline
\textbf{Input:} Data streams $\{(\mathbf{x}_1,\mathbf{y}_1), (\mathbf{x}_2,\mathbf{y}_2),\cdots\}$ that arrive \\
\ \ \ \ \ \ \ \ \ \ \ \ \ \ one sample each time; \\
 \ \ \ \ \ \ \ \ \ \ \ \ Parameters: $\alpha,\beta,\eta$, and the forgetting factor $\mu$;\\
\textbf{Initialize} $\mathbf{P}_0=\textbf{0}_{d\times m}$, $\mathbf{C}_{0,\mathbf{XX}}=\mathbf{0}_{d\times d}$,
$\mathbf{C}_{0,\mathbf{XY}}=\mathbf{0}_{d\times m}$, \\
\ \ \ \ \ \ \ \ \ \ \ \ \ \ \ \ $\mathbf{C}_{0,\mathbf{YY}}=\mathbf{0}_{m\times m}$, and $\Omega_0$ = $\Gamma_0$ =$\mathbf{I}_{m\times m}$;
\\
\textbf{Method}\\
1.\ \ \ $\textbf{for} \ \ t=1, 2, \cdots $ \\
2.\ \ \ \ \ \ \ \ $\mathbf{C}_{t,\mathbf{YY}}=\mu\mathbf{C}_{t-1,\mathbf{YY}}+\mathbf{y}_t\mathbf{y}_t^T$;\\
3.\ \ \ \ \ \ \ \ $\mathbf{C}_{t,\mathbf{XY}}=\mu\mathbf{C}_{t-1,\mathbf{XY}}+\mathbf{x}_t\mathbf{y}_t^T$;\\
4.\ \ \ \ \ \ \ \ $\mathbf{C}_{t,\mathbf{XX}}=\mu\mathbf{C}_{t-1,\mathbf{XX}}+\mathbf{x}_t\mathbf{x}_t^T$;\\
5.\ \ \ \ \ \ \ \ Update $\mathbf{P}_t$ based on (\ref{phat1}),(\ref{phat2}), and (\ref{p});\\
6.\ \ \ \ \ \ \ \ Update $\Omega_t$ based on (\ref{omegatt});\\
7.\ \ \ \ \ \ \ \ Update $\Gamma_t$ based on (\ref{gam});\\
8. \ \ \textbf{end}\\
\textbf{end Method} \\
\textbf{Output:} Regression coefficient matrix $\mathbf{P}_t\in \mathbb{R}^{m\times d}.$\\
\hline
\end{tabular}
\end{center}
\end{table}

\subsection{Theoretical Analysis}\label{theory}
Here we prove that our algorithm will converge and the limit of $\mathbf{P}_t$ is the true $\mathbf{P}^{\ast}$.
Suppose that the data arrive according to the rule $\mathbf{y}_n = \mathbf{P}^*\mathbf{x}_n$.
We will respectively prove that $\mathbf{P}_t$, $\Omega_t$ and $\Gamma_t$ converge in this case.

We propose some very mild conditions on the inputs $\mathbf{X}_t$'s. We assume that the inputs are bounded, and thus $\cxxt$ is bounded. We also assume that $\cxxt$ often is not singular, and there exists an $\delta > 0$ such that there are infinite many $t$ with the eigenvalues of $\cxxt$ being all greater than $\delta$. This is also a mild condition because $\cxxt$ often has full rank unless the inputs are strongly linearly dependent.

By the arguments above, we have obtained that
\begin{equation}
 \Omega_{t-1}\mathbf{P}_t+\alpha\Gamma_{t-1}\mathbf{P}_t \mathbf{C}_{t,\mathbf{XX}} =\Omega_{t-1}\mathbf{P}_{t-1}+\alpha \Gamma_{t-1}\mathbf{C}_{t,\mathbf{XY}}^T .\nonumber
\end{equation}
\begin{equation}
\Omega_t^{-1} = \frac{1}{\beta + \rho} \left( \beta \Omega_{t-1}^{-1} + \rho \mathbf{I} +  (\mathbf{P}_t- \mathbf{P}_{t-1}) (\mathbf{P}_t- \mathbf{P}_{t-1})^T \right)^{-1}. \nonumber  \label{omega}
\end{equation}
\begin{equation}
\Gamma_t = \left( \mathbf{I} + \frac{\eta}{\alpha} (\mathbf{P}^* - \mathbf{P}_{t-1}) \mathbf{C}_{t, \mathbf{XX}} (\mathbf{P}^* - \mathbf{P}_{t-1})^T \right)^{-1}.   \label{gamma}
\end{equation}

If we let $\at = \pt - \pstat$, we have
\beq
\at  + \alpha \Omega_{t-1}^{-1} \gal \at \cxxt = \atl. \label{A}
\eeq
We can look at the Frobenius norm of $\at, \atl$.
\begin{align}
 \|\atl \|_F^2   = & \| \at+ \alpha \Omega_{t-1}^{-1} \gal \at \cxxt\|_F^2 \nonumber \\
  =&  tr ( (\at+ \alpha \Omega_{t-1}^{-1} \gal \at \cxxt)   \nonumber\\
   &\times (\at+ \alpha \Omega_{t-1}^{-1} \gal \at \cxxt)^T )\nonumber\\
 =& tr(\at \at^T) + \| \alpha \Omega_{t-1}^{-1} \gal \at \cxxt\|_F^2  \nonumber\\
&+2 tr\left( \at \alpha \cxxt^T \at^T \Gamma_{t-1}^T {\Omega_{t-1}^{-1}}^T \right).
\end{align}
Therefore
 \begin{align}
 & \| \atl\|_F^2 - \| \at\|_F^2  \nonumber \\
 = & \| \alpha \Omega_{t-1}^{-1} \gal \at \cxxt\|_F^2  \nonumber \\
  & + 2 tr\left( \at \alpha \cxxt^T \at^T \Gamma_{t-1}^T {\Omega_{t-1}^{-1}}^T \right). \label{main}
 \end{align}
We want to study the eigenvalues of $\at \alpha \cxxt^T \at^T \Gamma_{t-1}^T {\Omega_{t-1}^{-1}}^T$, i.e., the eigenvalues of
$\alpha \Omega_{t-1}^{-1} \gal \at \cxxt \at^T$.
First notice that $\at \cxxt \at^T$ is a positive semi-definite Hermitian matrix as $\cxxt$ is.
Also notice that $\Omega_{t-1}^{-1}$ and $\gal$ are both positive definite Hermitian matrices.

We prove the following lemma.
\begin{lemma}\label{real}
If $\mathbf{A}$ is Hermitian positive definite and $\mathbf{B}$ is Hermitian, then $\mathbf{AB}$ has all real eigenvalues.
\end{lemma}
\begin{proof}
Since $\mathbf{A}$ is Hermitian positive definite, there exists $\mathbf{D}$ such that $\mathbf{A} = \mathbf{DD}^*$. So the eigenvalues of $\mathbf{AB}$ are the eigenvalues of $\mathbf{DD}^*\mathbf{B}$, and are also the eigenvalues of $\mathbf{D}^* \mathbf{BD}$. This is because the eigenvalues of $\mathbf{XY}$ are the eigenvalues of $\mathbf{YX}$.\footnote{This is a well-known fact that the eigenvalue of $\mathbf{XY}$ is an eigenvalue of $\mathbf{YX}$. This is because if $\mathbf{XYv} = \lambda \mathbf{v}$, then $\mathbf{YXYv = Y}\lambda \mathbf{v} = \lambda \mathbf{Yv}$. Then $\lambda$ is an eigenvalue of $\mathbf{YX}$ with eigenvector $\mathbf{Yv}$. And as a consequence of this, if $\mathbf{XY}$ has an eigenvalue decomposition, then $\mathbf{YX}$ also has eigenvalue decomposition, with the same set of eigenvalues, while the eigenvectors $\mathbf{v}$ for $\mathbf{XY}$ becomes $\mathbf{Yv}$ for $\mathbf{YX}$.}
In addition, since $\mathbf{D}^* \mathbf{BD}$ is also a Hermitian matrix, it has all eigenvalues being real. Hence, $\mathbf{AB}$ has all real eigenvalues.
\end{proof}
\begin{lemma}\label{pd}
If $\mathbf{A}$ is Hermitian positive definite, and $\mathbf{B}$ is positive definite. Then $\mathbf{AB}$ is positive definite.
\end{lemma}
\begin{proof}
Since $\mathbf{A}$ is positive definite Hermitian, we can write $\mathbf{A} = \mathbf{DD}^*$.
$\mathbf{B}$ being positive definite means that for any vector $\mathbf{v}$, we have $\langle \mathbf{Bv}, \mathbf{v} \rangle \geq 0$.

We know eigenvalues of $\mathbf{AB}$ are eigenvalues of $\mathbf{DD}^* \mathbf{B}$, and are also eigenvalues of $\mathbf{D}^* \mathbf{B D}$.
For $\mathbf{D}^* \mathbf{B D}$, we have that $\langle \mathbf{D}^* \mathbf{B Dv}, \mathbf{v} \rangle = \langle \mathbf{BDv}, \mathbf{Dv} \rangle = \langle \mathbf{B(Dv)}, \mathbf{(Dv)} \rangle \geq 0$. Therefore we have completed the proof.
\end{proof}

Since $\Omega_{t-1}^{-1}, \gal$ are both Hermitian positive definite, by Lemma \ref{pd},  $\Omega_{t-1}^{-1} \gal$ is positive definite. Together with the fact that $\at \cxxt \at^T$ is Hermitian positive definite, we have that by Lemma \ref{real}, $\Omega_{t-1}^{-1} \gal \at \cxxt \at^T$ has all positive real roots.


We thus have proved that in (\ref{main}), $2 tr\left( \at \alpha \cxxt^T \at^T \Gamma_{t-1}^T {\Omega_{t-1}^{-1}}^T \right) \geq 0$. So we have that $\| \at \|_F \leq \| \atl \|_F$.
Because $\| \at\|_F$ is monotone decreasing, we must have that it will converge to a value. If the limit value is 0, it means that our $\pt$ will eventually converge to the true $\mathbf{P}^*$, and we are done.

Otherwise, we can assume that $\| \at \|_F^2$ converges to a positive value $m^* > 0$. We will prove a contradiction in this case.

By (\ref{main}) again, we have that
\begin{align} &\| \at - \atl \|_F^2  \nonumber \\
= & \alpha^2 \| \Omega_{t-1}^{-1} \gal \at \cxxt \|_F^2 \nonumber\\
 \leq & \alpha^2 \| \at\|_F^2 \| \cxxt \Omega_{t-1}^{-1} \gal \|_F^2\label{31}.
\end{align}
The last inequality holds since for matrix norm, we have $\| \mathbf{AB} \| \leq \|\mathbf{A}\| \|\mathbf{B}\|$.

If we have a sequence of $t$ such that $\| \cxxt \Omega_{t-1}^{-1} \gal \|_F \to 0$ then we are done.
Because if this is the case, together with the fact that $\| \at \|_F$ is bounded (as it converges to $m^*$), (\ref{31}) will converge to 0; Since (\ref{31}) is an upper-bound for $\| \at - \atl \|_F$, we thus can conclude that $\| \at - \atl \|_F$ converges to 0.
This is equivalent to $\| \mathbf{P}_t- \mathbf{P}_{t-1}\|_F$ converges to 0.
To show that $\at$ converges to 0, we will first use the following lemma.
\begin{lemma}\label{omegai}
If $\| \at - \atl \|_F$ is bounded, then $\Omega_t$ converges to $\mathbf{I}$.
\end{lemma}

\begin{proof}
 We can rewrite (\ref{omegatt}) as
 \[ \Omega_t^{-1} - \mathbf{I} = \frac{\beta}{\beta + \rho} (\Omega_{t-1}^{-1} - \mathbf{I}) + \frac{1}{\beta + \rho}(\mathbf{P}_t- \mathbf{P}_{t-1})(\mathbf{P}_t- \mathbf{P}_{t-1})^T.\]
 Based on triangular inequality, we have that
 \begin{align}
  \|\Omega_t^{-1} - \mathbf{I} \|_F \leq &\frac{\beta}{\beta + \rho}  \|(\Omega_{t-1}^{-1} - \mathbf{I})\|_F \nonumber\\
 &+ \frac{1}{\beta \!+\! \rho}\|(\mathbf{P}_t \!-\! \mathbf{P}_{t-1})(\mathbf{P}_t \!-\! \mathbf{P}_{t-1})^T\|_F.\label{difference}
 \end{align}
 Since $\|(\mathbf{P}_t- \mathbf{P}_{t-1})(\mathbf{P}_t- \mathbf{P}_{t-1})^T\|_F$ converges to 0, we have that for any $\epsilon > 0$, we can find a $T$ large enough such that for all $t \geq T$, $\|(\mathbf{P}_t- \mathbf{P}_{t-1})(\mathbf{P}_t- \mathbf{P}_{t-1})^T\|_F \leq \epsilon$.   By iteratively applying (\ref{difference}), we have that
 \begin{align}
 \|\Omega_t^{-1} - \mathbf{I} \|_F \leq &\left( \frac{\beta}{\beta + \rho} \right)^{t- T} \|\Omega_T^{-1} - \mathbf{I} \|_F \nonumber\\
 &+ \frac{1}{\beta + \rho} \sum_{j = 1}^{t-T-1} \left( \frac{\beta}{\beta + \rho} \right)^{j} \epsilon.
 \end{align}
 As $t$ goes to infinity, we can see that $\left( \frac{\beta}{\beta + \rho} \right)^{t- T} \|\Omega_T^{-1} - \mathbf{I} \|_F$ goes to 0; by computing the sum of the power series $\frac{1}{\beta + \rho} \sum_{j = 1}^{t-T-1} \left( \frac{\beta}{\beta + \rho} \right)^{j} \epsilon$ we can see that it is bounded above by $\epsilon$ multiplying by a constant only depending on $\beta, \rho$. Since $\epsilon$ is arbitrary, we know that $\Omega_t^{-1}\!-\!\mathbf{I}$ will converge to $\mathbf{0}$.

 \end{proof}

Since $\| \cdot \|_F$ gives a metric for a complete normed space and $\| \mathbf{A_t} - \mathbf{A_{t-1}}\|_F$ converges to 0, we have that this Cauchy sequence $\mathbf{A}_t$'s will converge to a limit $\mathbf{A}$. Similarly, $\mathbf{P}_t$'s converge to a limit.
By (\ref{A}), we have that
\begin{align}
 \|  \mathbf{A_t} \!-\! \mathbf{A_{t-1}}\|_F \!\geq\! \alpha \sigma_{min}(\Omega_{t-1}^{-1}) \sigma_{min}(\gal) \sigma_{min}(\cxxt) \|\mathbf{A}_t\|_F. \label{bound}
 \end{align}
As $\cxxt$ has singular values bounded away from 0 and $\mathbf{P}_t$'s converge, we have that $\sigma_{min}(\gal)$ is bounded away from 0. Because of $\cxxt, \gal, \oml^{-1}$ all having eigenvalues bounded away from 0, we have that by (\ref{bound}), $\|\mathbf{A}_t\|_F$ converges to 0.

%

Now suppose otherwise that $\| \cxxt \Omega_{t-1}^{-1} \gal \|_F$ does not converge to 0. Then it is relatively large.
Notice by (\ref{main}), we have that
\[ \| \atl \|_F^2 - \| \at \|_F^2 \geq \alpha^2 \| \Omega_{t-1}^{-1} \gal \at \cxxt \|_F^2. \]
The right hand side satisfies
\[ \| \Omega_{t-1}^{-1} \gal \at \cxxt \|_F \geq \sigma_{min} (\cxxt \Omega_{t-1}^{-1} \gal ) \| \at\|_F \]
where $\sigma_{min}$ denotes the minimal singular value.

We will show that $\sigma_{min} (\Omega_{t-1}^{-1} \gal \at \cxxt)$ is bounded away from 0.
This is because
\begin{align}
& \sigma_{min} (\cxxt \Omega_{t-1}^{-1} \gal ) \nonumber\\
\geq & \sigma_{min} (\Omega_{t-1}^{-1}) \sigma_{min}(\gal) \sigma_{min}(\cxxt) \nonumber \\
\geq & \sigma_{min} (\gal) \sigma_{min} (\cxxt). \label{ine}
\end{align}
The last inequality holds because $\oml$ always has eigenvalues between 0 and 1.
We now look at the eigenvalues of $\gal$.  By (\ref{gamma}), we have that
$ \Gamma_{t-1}^{-1} = \mathbf{I} + \alpha \atl \cxxt \atl^T.$ Thus,
\[ \| \Gamma_{t-1}^{-1} \|_F \leq \| \mathbf{I}\|_F + \alpha \| \atl \|_F \| \cxxt \|_F \| \atl\|_F.\]
Since $\| \atl\|_F$ is bounded as it converges to a limit, and we have assumed that $\cxxt$ is bounded, we have that $\| \Gamma_{t-1}^{-1}\|_F$ is bounded from above.
Hence the largest eigenvalue of $\Gamma_{t-1}^{-1}$ is bounded from above, and the minimal eigenvalue of $\gal$ is bounded away from 0.

Now in (\ref{ine}),  since $\sigma_{min} (\gal)$ is bounded away from 0 and there are sufficient number of $t$ such that the eigenvalues of $\sigma_{min} (\cxxt)$ are not small, we that $\| \at \| - \| \atl\|$ is bounded away from 0. However, since $\| \at\|$ will eventually converge to a limit, it means $\| \at \|  - \| \atl \|$ converges to 0, contradiction.

After showing the convergence together with an lower bound for the convergence rate of $\| \at \|_F$, to complete the analysis, we will also show that $\Omega_t$ and $\Gamma_t$ converges to $\mathbf{I}$.

Since $\mathbf{P}_t$ converges to $\mathbf{P}^*$, it is clear that $\| \at - \atl \|_F$ is bounded; thus the conditions in Lemma \ref{omegai} are satisfied. Lemma \ref{omegai} shows that $\Omega_t$ converges to $\mathbf{I}$.

We are left to show that $\Gamma_t$ converges to $\mathbf{I}$. We will use formula (\ref{gamma}):
\[ \Gamma_t = \left( \mathbf{I} + \frac{\eta}{\alpha} (\mathbf{P}^* - \mathbf{P}_{t-1}) \mathbf{C}_{t, \mathbf{XX}} (\mathbf{P}^* - \mathbf{P}_{t-1})^T \right)^{-1}. \nonumber \]
The above equation implies
\begin{equation}  \| \Gamma_t^{-1} - \mathbf{I}  \|_F  = \frac{\eta}{\alpha} \| (\mathbf{P}^* - \mathbf{P}_{t-1}) \mathbf{C}_{t, \mathbf{XX}} (\mathbf{P}^* - \mathbf{P}_{t-1})^T \|_F. \label{gammabound}
\end{equation}
We know that $(\mathbf{P}^* - \mathbf{P}_{t-1})^T $ converges to 0, and that $\cxxt$ is bounded, thus the right hand side of (\ref{gammabound}) converges to 0. Since $\| \Gamma_t^{-1} - \mathbf{I}  \|_F$ is bounded above by the right hand side of (\ref{gammabound}), it is clear that $\| \Gamma_t^{-1} - \mathbf{I}  \|_F$ converges to 0; and thus $\Gamma_t$ converges to $\mathbf{I}$.

To conclude, we have proven that under a very mild condition that the inputs are not highly dependent on each other and that the inputs can be normalized, then if the incoming data satisfies a clear linear equation, our analysis shows that the learned $\mathbf{P}$ will converge to the true underlying $\mathbf{P}$ value with a satisfactory convergence rate. The learned $\Gamma, \Omega$ will be $\mathbf{I}.$

The above analysis is based on the assumption that $\mathbf{y} = \mathbf{P}^* \mathbf{x}$. However, in real life applications, it is always the case that the underlying distribution is not ideal. For this purpose, suppose each dimension of the output has an error term $\epsilon_{t,i}$, i.e., we assume that the underlying distribution follows
 \[ \mathbf{y}_t = \mathbf{P}^* \mathbf{x}_t + \overrightarrow{\epsilon_t}. \]
 We further assume that $\overrightarrow{\epsilon_t}$ are independent for different $t$'s, and for each dimension $i$, the bias is 0, i.e., the expectation $\E[\epsilon_{t,i}] = \mathbf{0}$.

 Similar to the computation as in (\ref{compute}), we would have
 \begin{align}
& \E[ \by_t \by_t^T - \by_t \bx_t^T \pt^T - \pt^T \bx_t \by_t^T + \pt \bx_t \bx_t^T \pt^T \nonumber]\\
= & \E [ ({\mathbf{P}_t^*} \bx_t + \overrightarrow{\epsilon_t})({\mathbf{P}_t^*} \bx_t + \overrightarrow{\epsilon_t})^T] - \E[({\mathbf{P}_t^*} \bx_t + \overrightarrow{\epsilon_t}) \bx_t^T \pt^T] \nonumber \\
  & -\E[\pt^T \bx_t ({\mathbf{P}_t^*} \bx_t + \overrightarrow{\epsilon_t})^T ] + \E[\pt \bx_t \bx_t^T \pt^T]. \label{compute2}
  \end{align}
Since $\E[\epsilon_t] = \mathbf{0}$, we have that (\ref{compute2}) can be rewritten as
\begin{align}
 =& \E[{\mathbf{P}_t^*} \bx_i \bx_t^T {\mathbf{P}_t^{*T}} \!-\! {\mathbf{P}_t^*} \bx_t \bx_t^T \pt^T \!-\! \pt \bx_t \bx_t^T {\mathbf{P}_t^*}^T \!+\! \pt \bx_t \bx_t^T \pt^T] \nonumber \\
  & + \E[ \overrightarrow{\epsilon_t} \overrightarrow{\epsilon_t}^T] \nonumber\\
 = & \E[{\mathbf{P}_t^*} \bx_t \bx_t^T ({\mathbf{P}_t^{*T}} \!-\! \pt^T)  \!-\! \pt \bx_t \bx_t^T ({\mathbf{P}_t^{*T}} \!-\! \pt^T)] \!+\!  \E[ \overrightarrow{\epsilon_t} \overrightarrow{\epsilon_t}^T]\nonumber\\
 = & \E[({\mathbf{P}_t^*} -\mathbf{P}_t) \bx_t \bx_t^T ({\mathbf{P}_t^*} -\mathbf{P}_t)^T] + \text{Cov}(\overrightarrow{\epsilon}). \label{compute3}
\end{align}

Therefore by (\ref{gam}), we have that when $\Gamma_t$ converges, it will be \[\Gamma^{-1} = \mathbf{I} + \frac{\eta}{\alpha} \text{Cov}(\overrightarrow{\epsilon}). \] When there are correlations between the noise in different channels, the limit of $\Gamma$ will reveal the correlation.


\subsection{Time Complexity Analysis}
In Algorithm 1, the most time-consuming part of MORES is to update $\mathbf{P}_t$, $\Omega_t$, and $\Gamma_t$, and the time cost of other parts can be ignored.
For updating $\mathbf{P}_t$, the complexity is $O(m^3 + m^2 d + m d^2)=O(m^3+ md^2)$.
Updating $\Gamma$ needs $O(m^3+d^2m)$. It costs $O(m^3 + m^2d)$ for updating $\Omega_t$.
Therefore, the total time complexity of MORES is of order $O(m^3 + md^2)$. This is a low-order polynomial in $m$ and $d$. The three updates takes abound the same amount of complexity.
\section{Experiments}
To evaluate the performance of MORES, we perform the experiments on both one synthetic dataset and three real-world datasets:  the Barrett WAM dataset \cite{NguyenTuong}, the stock price dataset,
and the weather dataset \cite{alvarez2008sparse}.
We compare MORES with an online multiple-output regression method, iS-PLS \cite{rothman2010sparse}, two online multi-task learning methods, ELLA\footnote{The implementation can be downloaded from {http://occam.olin.edu/node/2\#publications}.} \cite{eaton2013ella} and OMTL\footnote{We use $l_2$ norm as the loss function of OMTL to enable it for regression tasks.} \cite{saha2011online}. We also compare with two variants of PA algorithm in \cite{crammer2006online} called PA-I and PA-II, which are two classical online learning approaches for single regression tasks. We name our simple formulation of online multiple-output regression proposed at the beginning of Sect. 2 as SOMOR for short.

In the experimental study, we mainly focus on the accuracy of online regression prediction to evaluate our model's quality.
The popular metric, Mean Absolute Error (MAE), is used to measure the prediction quality. MAE is defined as: MAE $=\frac{1}{t}\sum_{i=1}^t|y_i-\widehat{y}_i|$, where $\widehat{y}_i$ denotes the estimated values of the $i$-th instance, and ${y}_i$ is the true response values.

There are some parameters to be set in advance. The parameters $\beta$ and $\eta$ are always set to 1 and 100 throughout the experiments, respectively (we found when $\beta =1, \eta = 100$, the performance was consistently good on all the datasets). The rest parameters are tuned from the space $\{ 10^{-2}, 10^{-1},\ldots, 10^3, 10^4\}$. We first use the first 100 samples from each dataset to tune the parameters, and then apply the optimal parameters to the rest samples for online multiple-output regression. To conduct fair experimental comparisons, the parameters in the other methods are tuned from the same search space with MORES.

\subsection{Synthetic Dataset}
\begin{figure}
\centering
\includegraphics[width=0.8\linewidth]{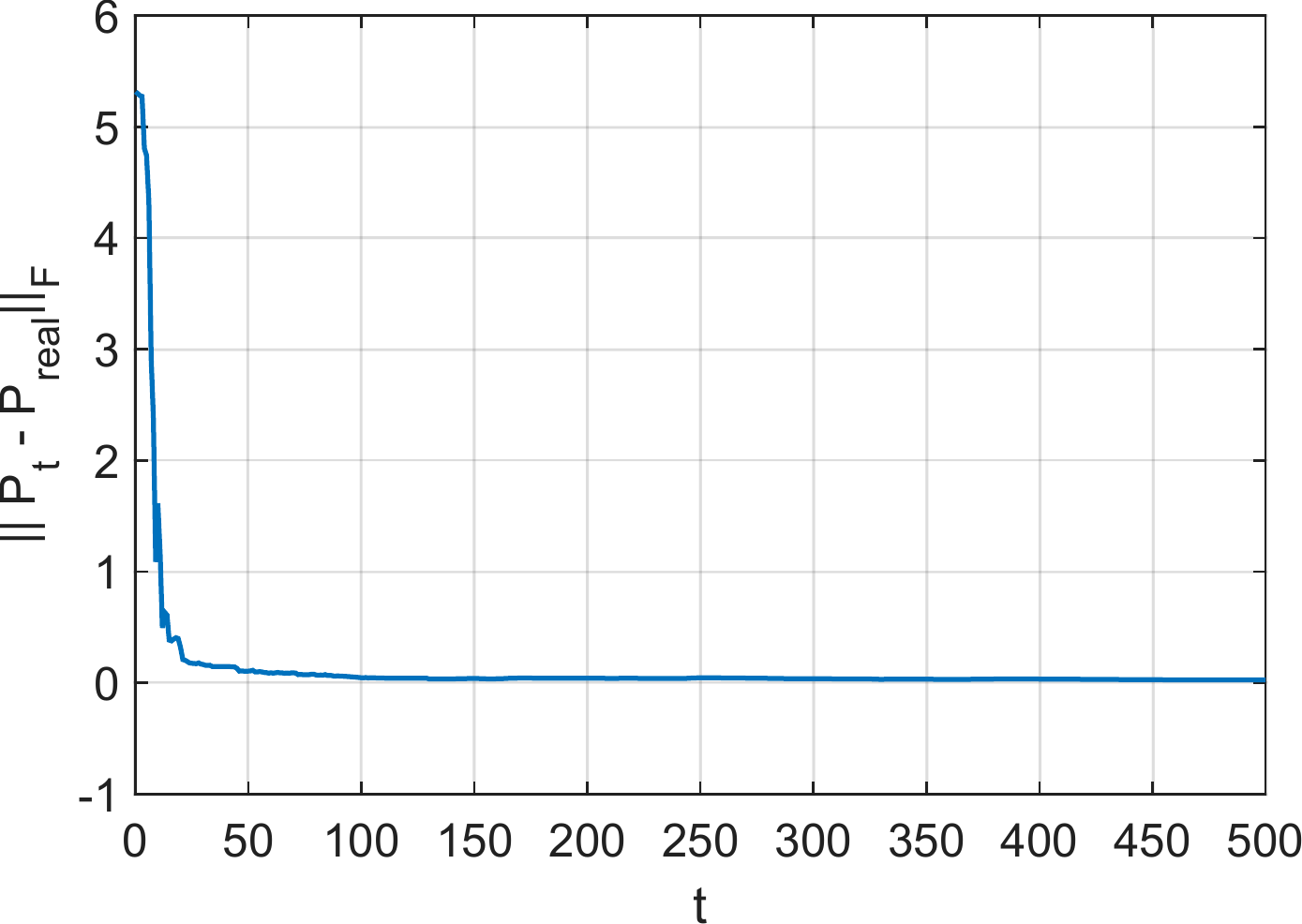}
\caption{Convergence study of the learned $\mathbf{P}_t$, where $t$ represents the total samples collected currently.}
\label{coverg}
\end{figure}

\begin{table}
\footnotesize
\caption{The learned correlation coefficient matrices at different $t$.}
\label{correlation}
\centering
\subtable[ The correlation matrix of the coefficients change at $t=50$  ]{
\begin{tabular}{c|c|c|c}
\hline
{Me} & 1st        &  2nd    &  3rd       \\
\hline
{PA}  &      1.00  & 0.20  &  0.83   \\
\hline
{PA} &       0.20 &  1.00 &  0.52   \\
\hline
{iS} &     0.83  & 0.52  &  1.00   \\
\hline
\end{tabular}
\label{tab:firsttable}
}
\qquad
\subtable[ The coefficient matrix of the residual errors at $t=50$ ]{
\begin{tabular}{c|c|c|c}
\hline
{Met} & 1st        &  2nd    &  3rd      \\
\hline
{PA}  &     1.00  &  0.04  &  0.51    \\
\hline
{PA} &        0.04   & 1.00   & 0.62  \\
\hline
{iS} &     0.51  &  0.62  &  1.00  \\
\hline
\end{tabular}
\label{tab:secondtable}
}

\subtable[The correlation matrix of the coefficients change at $t=500$  ]{
\begin{tabular}{c|c|c|c}
\hline
{Me} & 1st        &  2nd    &  3rd       \\
\hline
{PA}  &      1.00  & 0.01  &  0.11   \\
\hline
{PA} &       0.01 &  1.00 &  0.04   \\
\hline
{iS} &     0.11  & 0.04  &  1.00   \\
\hline
\end{tabular}
\label{tab:firsttable}
}
\qquad
\subtable[ The correlation matrix of the residual errors at $t=500$ ]{
\begin{tabular}{c|c|c|c}
\hline
{Met} & 1st        &  2nd    &  3rd      \\
\hline
{PA}  &     1.00  &  0.07  &  0.60   \\
\hline
{PA} &        0.07   & 1.00   & 0.60  \\
\hline
{iS} &     0.60  &  0.60  &  1.00  \\
\hline
\end{tabular}
\label{tab:secondtable}
}
\end{table}
We first generate an artificial dataset to conduct a ``proof of concept'' experiment before we perform experiments on real datasets. The synthetic dataset is generated as follows. First, we generate 500 samples $\mathbf{x}_i\in \mathbb{R}^{11}$ as the input vectors. The first 10 entry in the input vector is drawn from the standard normal distribution, and the last entry is 1 in order to learn the bias term. Then we generate two weight vectors $\mathbf{p}_1\in \mathbb{R}^{11}$ and $\mathbf{p}_2\in \mathbb{R}^{11}$  with each entry sampled from the standard normal distribution, and form two regression tasks by $y_{i,1}=\mathbf{p}_1^T\mathbf{x}_i+\epsilon_{i,1}$ and $y_{i,2}=\mathbf{p}_2^T\mathbf{x}_i+\epsilon_{i,2}$. $\epsilon_{i,1}$ and $\epsilon_{i,2}$ are the noise drawn from the normal distribution with mean 0 and standard deviation 0.1. Finally, we construct the third regression task by $y_{i,3}=y_{i,1}+y_{i,2}+\epsilon_{i,3}=\mathbf{p}_1^T\mathbf{x}_i+\mathbf{p}_2^T\mathbf{x}_i+\epsilon_{i,1}+\epsilon_{i,2}+\epsilon_{i,3}$, where $\epsilon_{i,3}$ is the noise sampled from the normal distribution with mean 0 and standard deviation 0.1. Based on this, we know there are correlations between the noise (e.g., the first channel and the third channel).

\begin{figure*}
\centering
\subfigure[IBM]{\includegraphics[width=0.325\linewidth]{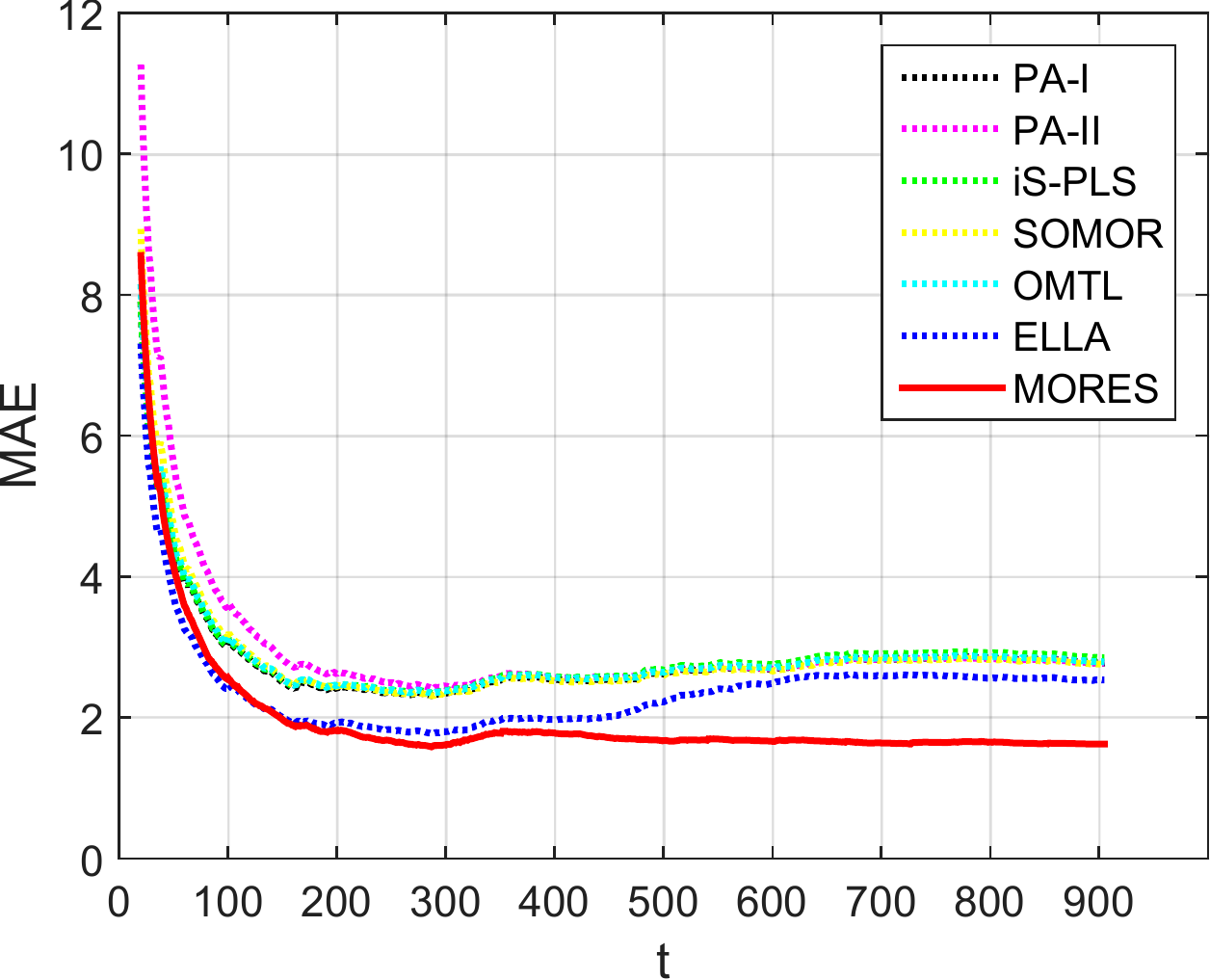}}
\subfigure[Yahoo]{\includegraphics[width=0.325\linewidth]{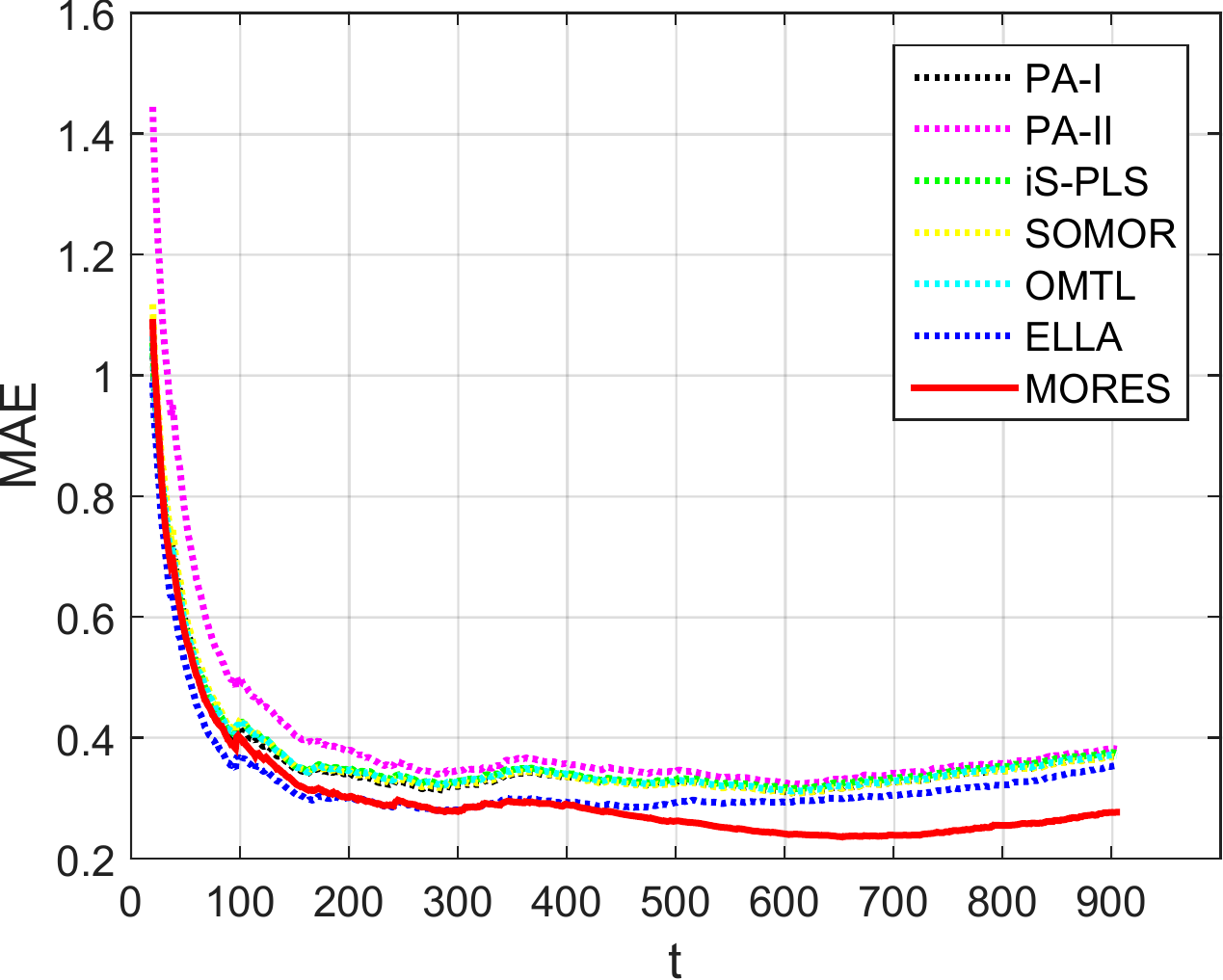}}
\subfigure[Microsoft]{\includegraphics[width=0.325\linewidth]{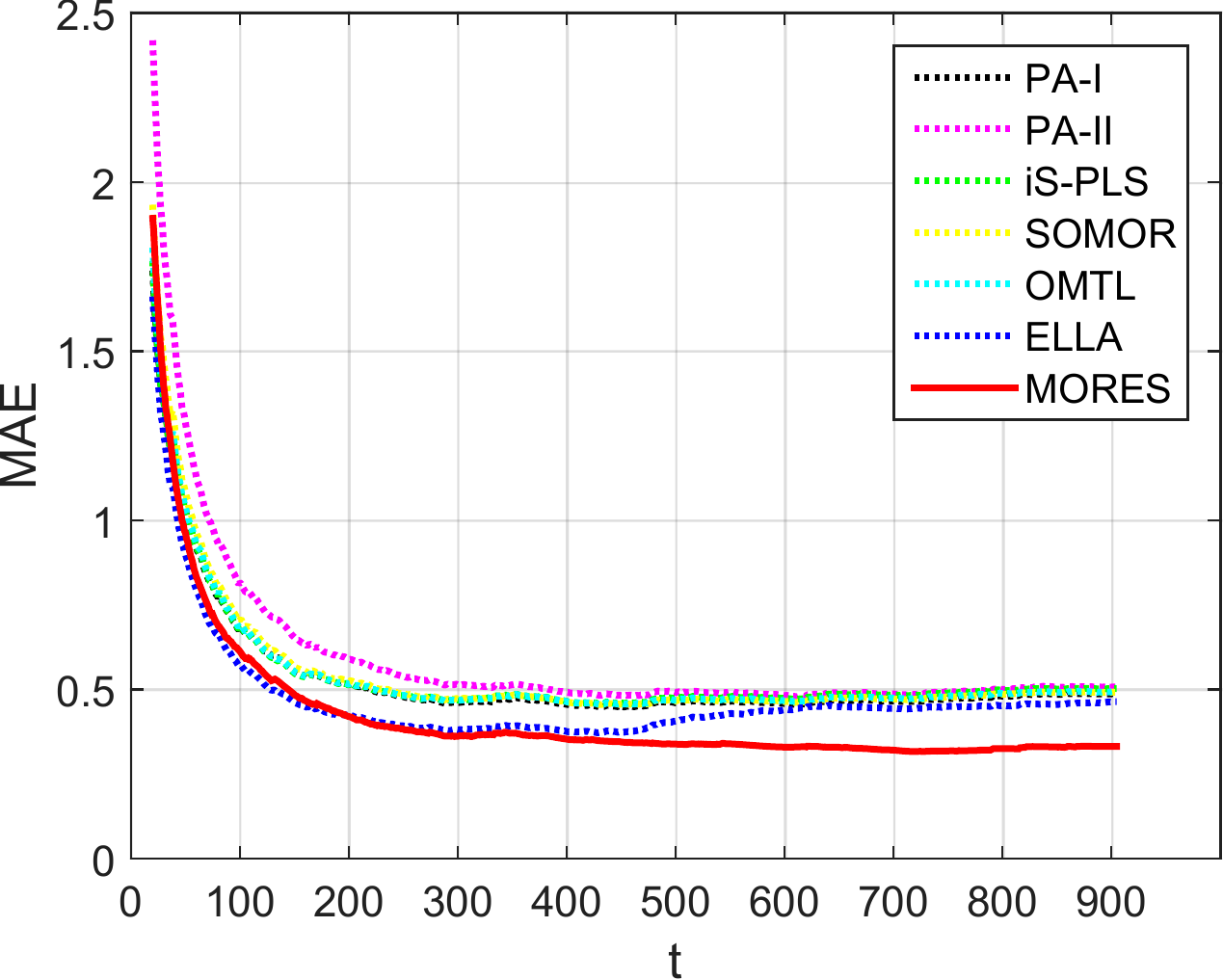}}
\subfigure[Apple]{\includegraphics[width=0.325\linewidth]{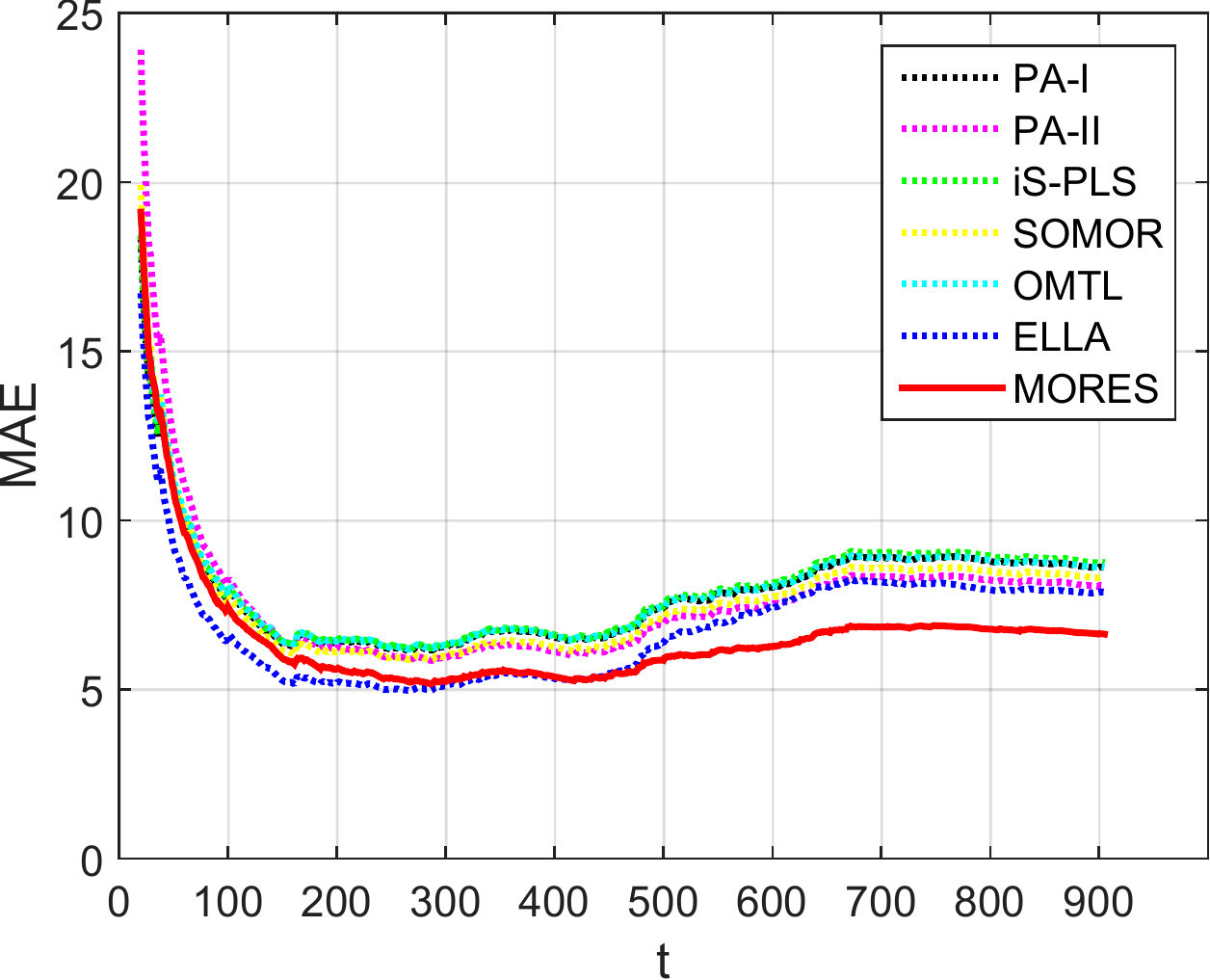}}
\subfigure[Oracle]{\includegraphics[width=0.325\linewidth]{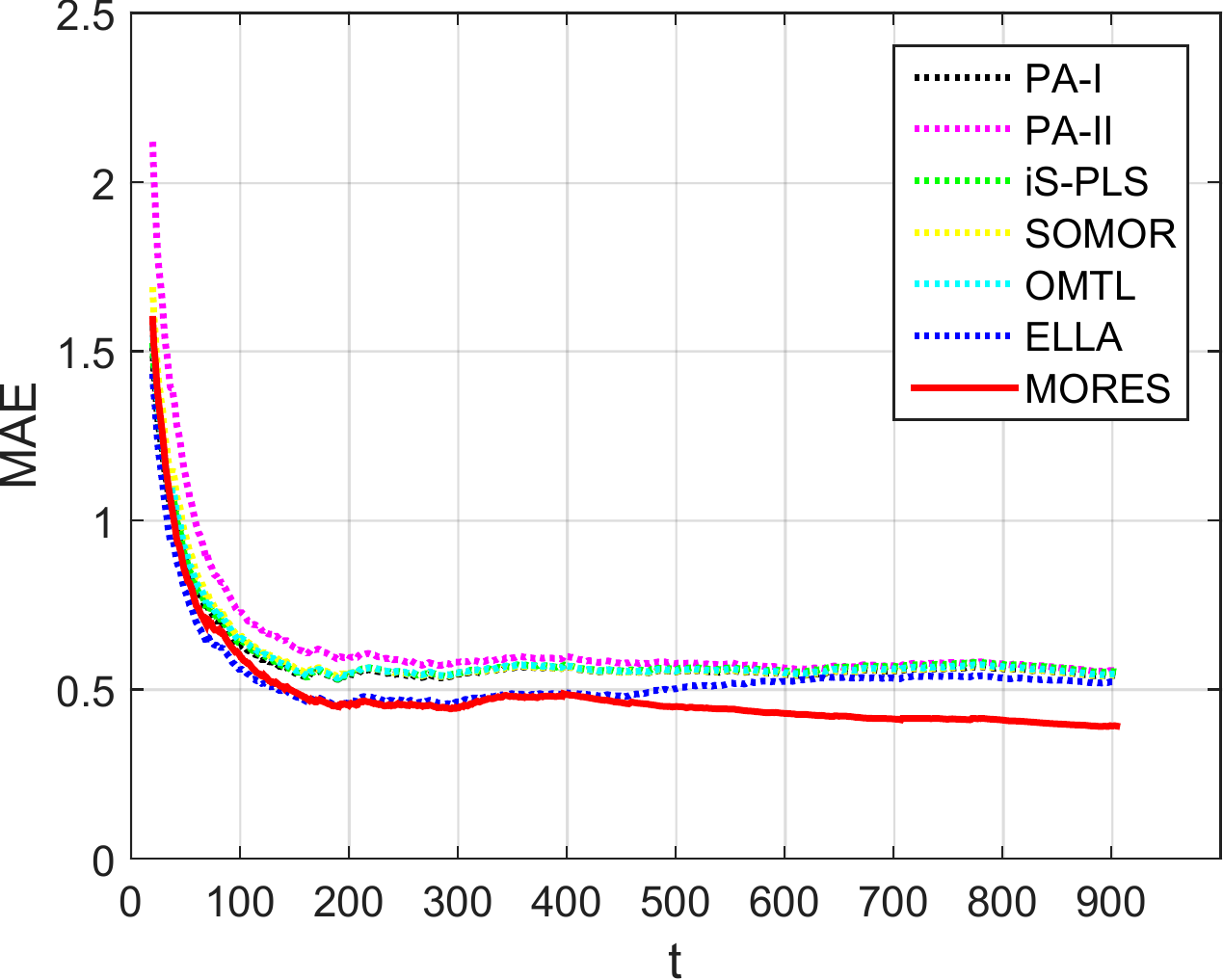}}
\subfigure[Average]{\includegraphics[width=0.325\linewidth]{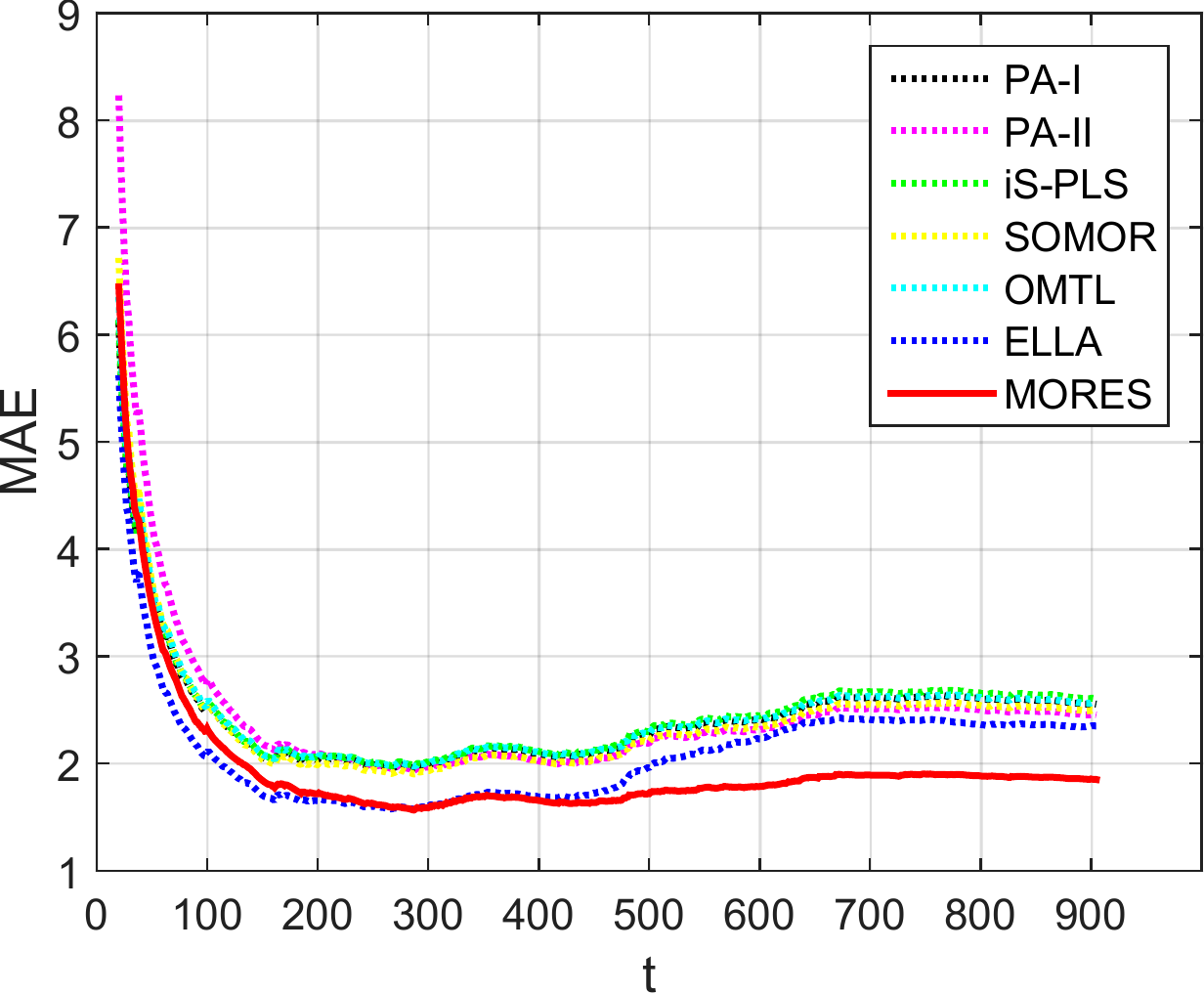}}
\caption{MAEs of different approaches as a function of the number of samples $t$ on different companies.}
\label{stock:iteration}
\end{figure*}

In order to demonstrate that our proposed algorithm can make the learn $\mathbf{P}_t$ eventually converge to the real model $\mathbf{P}_{real}=[\mathbf{p}_1,\mathbf{p}_2,\mathbf{p}_1+\mathbf{p}_2]$, we use the Frobenius norm $\|\cdot\|_F$ to measure the distance between $\mathbf{P}_t$ and $\mathbf{P}_{real}$ at each round $t$. The result is shown in Fig. \ref{coverg}. When $t=100$, the distance between $\mathbf{P}_t$ and $\mathbf{P}_{real}$ is close to zero, i.e., $\mathbf{P}_t$ converges to $\mathbf{P}_{real}$ at that step.

In addition, we also show that our algorithm can dynamically learn the relationship both the coefficients changes and the residual errors as samples continuously increase. Table \ref{correlation} lists the results. Table \ref{correlation} (a) and (b) are the learned correction matrix of the coefficients change and the correction matrix of the residual errors respectively, when $t=50$ samples arrive. Based on Fig. \ref{coverg}, we know that $\mathbf{P}_t$ has not yet converged to $\mathbf{P}_{real}$ at $t=50$, i.e., there should be correlations among the changing regression coefficients. Meanwhile, there should be also correlations among the residual errors. From Table \ref{correlation} (a) and (b), the relationships can be learned by our algorithm. When $t=500$, $\mathbf{P}_t$ converges to $\mathbf{P}_{real}$. Thus, there are no correlation among the coefficients when $\mathbf{P}_t$ is updated, while there are still correlations existing in the residual errors. Table \ref{correlation} (c) and (d) demonstrate that our algorithm can learn such relationships.

\subsection{Stock Price Prediction}

Following previous studies in \cite{lozano2013robust} and \cite{sohn2012joint}, we apply our algorithms to the stock data of companies for price prediction.
We choose the daily stock price data of five companies including IBM, Yahoo, Microsoft, Apple, and Oracle in the period from 2010 to 2013. The learned model can predict the stock prices in the future by using the stock prices in the past as inputs. Like \cite{lozano2013robust} and \cite{sohn2012joint}, we use the autoregressive model ${\mathbf{y}}_{t+1}=\mathbf{P}_{t}\mathbf{y}_{t}+\epsilon_t$, where ${\mathbf{y}}_{t+1}$ represents the real stock prices of the five companies at time $t+1$, and $\mathbf{P}_{t}$ denotes the learned regression coefficient matrix at time $t$.

The experimental results are reported in Table \ref{stock_price}. MORES achieves the best performance compared to the other methods. MORES gains 36.0$\%$, 21.8$\%$, 28.7$\%$, 16.0$\%$, and 24.9$\%$ relative accuracy improvement over ELLA, the second best approach, for IBM, Yahoo, Microsoft, Apple, and Oracle, respectively. Meanwhile, MORES obtains 21.4$\%$ relative improvement in terms of the average MAE over ELLA. These results show that dynamically learning the structures of both the regression coefficient matrix and the residual error vector, as well as utilizing the historical data in an appropriate way, is good for online multiple-output regression.
\begin{table}
\centering
\caption{MAEs of different methods on the stock price dataset. The last column is the average MAE. Best results are highlighted in bold fonts.}
\begin{tabular}{|c|c|c|c|c|c|c|c|}
\hline
{Method}      & IBM        &  Yahoo    &  Microsoft     &  Apple    & Oracle & Average   \\
\hline
{PA-I}  &    2.773  &  0.369 &   0.487  &  8.601 &   0.541 &   2.554 \\
\hline
{PA-II} &      2.771  &  0.381  &  0.507   & 8.057   & 0.556  &  2.454 \\
\hline
{iS-PLS} &      2.850  &  0.374 &   0.503 &   8.744 &   0.552   & 2.605\\
\hline
SOMOR &        2.755  &  0.368&    0.492  &  8.306  &  0.541  &  2.492\\
\hline
OMTL &    2.790  &  0.370  &  0.494 &   8.616   & 0.544 &   2.563 \\
\hline
ELLA &    2.533  &  0.354  &  0.464  &  7.876  &  0.522   & 2.350 \\
\hline
MORES &     \textbf{1.620}  &  \textbf{0.277} &   \textbf{0.331} &   \textbf{6.613} &   \textbf{0.392}   & \textbf{1.847}\\
\hline
\end{tabular}
\label{stock_price}
\end{table}

We also investigate the performances of different methods as a function of the number of samples ($t$). At the end of each online round, we calculate the MAE for each output attained so far. Fig.~\ref{stock:iteration} shows the results. The performance of MORES is superior to those of the other methods, especially when $t$ is larger. In addition, the MAE curves of Fig. \ref{stock:iteration} (a), (b), and (d) rise after falling as $t$ increases. This is because the stock price is severely evolving at the inflection point, such that the current model can not predict the future price well. Although the data are evolving, our algorithm is still better than the other methods under this circumstance. From Fig. \ref{stock:iteration} (a), we can see MORES can quickly adjust the model to fit in the data's evolvement.

We further verify the effectiveness of dynamically learning the structures of the coefficients change and the residual error vector, respectively. The experimental setting is as follows: we first set $\Omega$ to the identity matrix and update $\Gamma$ on each round, which indicates that only the \underline{R}elationships of the \underline{R}esidual \underline{E}rrors are learned in the process of model's update, called RRE. Second, we update $\Omega$ on each round and set $\Gamma$ to the identity matrix. This means that we only learn the \underline{R}elationships of the \underline{C}oefficients \underline{C}hange. We name it RCC for short. Last, we set both $\Omega$ and $\Gamma$ to the identity matrix, showing \underline{W}ithout \underline{R}elationship \underline{L}earning. We call it WRL.
Table \ref{component} lists the average MAEs of all the outputs with round $t$ increasing. Both RRE and RCC are better than WRL, which shows that dynamically learning the structures of the regression coefficients change and the residual errors are both beneficial to online regression. MORES achieves the best performance, showing that the combination of RRE and RCC is effective for predicting.
\begin{table}
\centering
\caption{The average MAE vs. round $t$. Best results are highlighted in bold fonts.}
\begin{tabular}{|c|c|c|c|c|c|c|}
\hline
{Method}      & $t\!=\!100$        &  $t\!=\!200$    &  $t\!=\!300$     &  $t\!=\!400$    & $t\!=\!500$    \\
\hline
WRL &    2.370 &  1.742 &   1.602   & 1.667  &  1.727  \\
\hline
RRE &    2.331 &  1.724 &   1.594   & 1.652  &  1.722  \\
\hline
RCC &    2.321  &  1.717  &  1.586  &  1.654  &  1.717   \\
\hline
MORES &     \textbf{2.310}  &  \textbf{1.701} &   \textbf{1.575} &   \textbf{1.641} &   \textbf{1.708}   \\
\hline
\end{tabular}
\label{component}
\end{table}
\begin{table*}
\centering
\caption{MAEs of different methods on the Barrett WAM dataset. The last column is the average MAE. Best results are highlighted in bold fonts.}
\begin{tabular}{|c|c|c|c|c|c|c|c|c|c|}
\hline
Method & 1st DOF       &  2nd DOF   &  3rd   DOF  &  4th  DOF&  5th DOF &  6th DOF &  7th DOF & Average   \\
\hline
PA-I  &      1.156  &  0.645  &  0.342 &   0.282   & 0.098  &  0.081  &  0.043   & 0.378  \\
\hline
PA-II &      1.192  &  0.729  &  0.332  &  0.294  &  0.094   & 0.079  &  0.042   & 0.394 \\
\hline
iS-PLS &     \textbf{0.421}   & 0.727  &  0.247  &  0.163   & \textbf{0.038}   & {0.047}   & \textbf{0.020}    &0.238 \\
\hline
SOMOR &       1.147 &    0.640 &    0.332  &   0.278 &    0.094  &   0.079  &   0.042 &    0.373 \\
\hline
OMTL  &     1.151  &   0.642  &   0.326   &  0.285  &   0.091  &   0.079  &   0.041  &   0.374\\
\hline
ELLA  &     0.470  &  0.774  &  0.282   & 0.277  &  0.087  &  0.085   & 0.044  &  0.288\\
\hline
MORES &   0.511  &  \textbf{0.294}  &  \textbf{0.160}  &  \textbf{0.141}  &  0.040  &  \textbf{0.042}  &  0.023  &  \textbf{0.173}\\
\hline
\end{tabular}
\label{robot}
\end{table*}
\subsection{Robot Inverse Dynamics}

We also study the problem of online learning the inverse dynamics of a 7 degrees of freedom of robotic arms on the Barrett WAM dataset.
This dataset consists of a total of 16,200 samples, where each sample is represented by 21 input features, corresponding to seven joint positions, seven joint velocities and seven joint accelerations. Seven joint torques for the seven degrees of freedom (DOF) are used as the outputs.

We summarize the results of different methods in Table \ref{robot}. For each output, MORES attains better prediction performance than the other methods except iS-PLS. iS-PLS has the best prediction performance on certain outputs, but our method beats it on most of the outputs. In addition, MORES significantly outperforms the two online multi-task learning algorithms, ELLA and OMTL, by learning the structures of the outputs and fully utilizing the historical data. MORES obtains 39.9\% and 52.9\% relative error deduction over ELLA and OMTL in the average MAEs, respectively.

\begin{figure}
\centering
\includegraphics[width=0.8\linewidth]{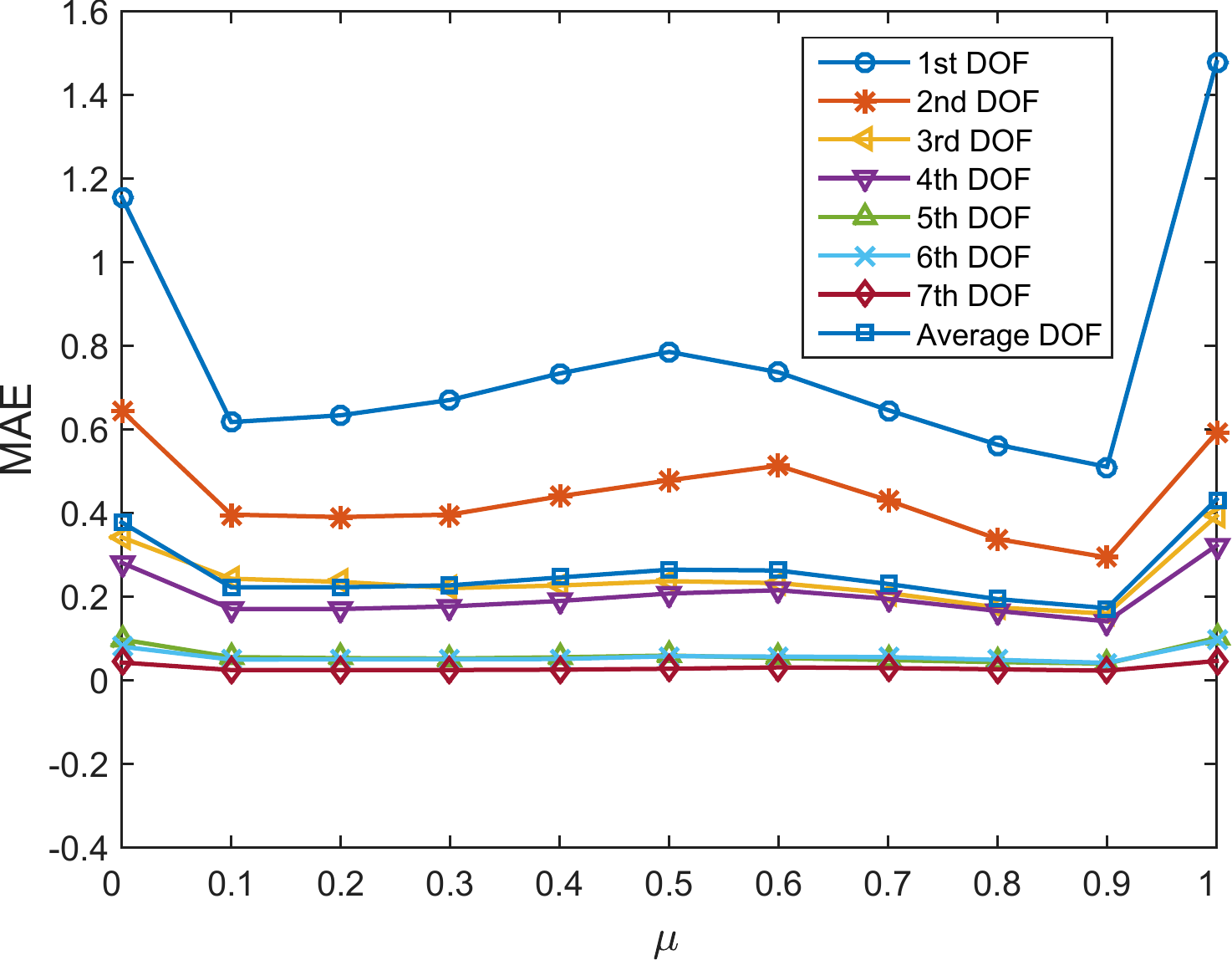}
\caption{MAEs of our method with different forgetting factors $\mu$.}
\label{factor}
\end{figure}

We verify the effectiveness of the forgetting factor $\mu$ in the experiment. We conduct the experiments with different values of $\mu$. Fig. \ref{factor} shows the results. Based on previous analysis in Sect. 2, we know no historical data are utilized to update the model on each round if $\mu=0$.
When $\mu = 1$, all the samples have the same weight for calculating the prediction loss.
As can be seen in Fig. \ref{factor}, when $0.1 \leq \mu \leq 0.9$ , the performance is improved compared to that of $\mu=0$.
This shows that taking advantage of the historical samples is good for online regression.
Moreover, when $0.1 \leq \mu \leq 0.9$ , the performance is improved compared to that of $\mu=1$. It implies that the data in this dataset are indeed evolving. By introducing $\mu$ to set higher weights on the newer training samples, the model can adapt to the data stream's evolvement, and the prediction accuracy can be improved.
In addition, when $\mu = 0.9$, MORES achieves the best performance.

\subsection{Weather Forecast}
We also evaluate our algorithms on the weather dataset for weather forecast \cite{alvarez2008sparse}. This dataset consists of wind speed, wind direction, barometric pressure, water depth, maximum gust, maximum wave height, air temperature, water temperature and average wave height, which is collected every five minutes by a sensor network located on the south coast of England. One and a half years' data containing 143,034 samples are used in the experiments. The first five variables are used as the predictors, and the rest are the response variables.

The experimental results are reported in Table \ref{weather}. MORES obtains better prediction performance than the other methods for almost all of the response variables. Meanwhile, the results of our algorithm are superior to those of the other methods in terms of the average MAE. Specifically, MORES achieves 20.4\%, 19.5\%, 57.1\%, 19.8\%, 20.4\%, 7.7\% relative accuracy improvement in terms of the average MAE over PA-I, PA-II, iS-PLS, SOMOR, OMTL, and ELLA, respectively.
\begin{table}
  \centering
\caption{MAEs of different methods on the weather dataset. `MWH', `AT', `WT', and `AWH' denote maximum wave height, air temperature, water temperature, and average wave height, respectively. The last column is the average MAE. Best results are highlighted in bold fonts.}
  \begin{tabular}{|c|c|c|c|c|c|c|}
\hline
{Method}      & MWH        &  AT    &  WT     &  AWH    & Average   \\
\hline
{PA-I}  &      0.770  &  0.217  &  0.144 &   0.007  &  0.285  \\
\hline
{PA-II} &       0.766&    0.216 &   0.141  &  0.006  &  0.282 \\
\hline
{iS-PLS} &       0.780 &   0.633 &   0.697  &  0.007 &   0.529 \\
\hline
{SOMOR} &      0.772 &   0.215  &  0.141  &  0.006 &   0.283\\
\hline
OMTL &    0.777   & 0.215  &  0.140  &  0.006  &  0.285 \\
\hline
ELLA &  0.802 &  \textbf{0.116}  & 0.059 &   0.007  &  0.246 \\
\hline
{MORES} &   \textbf{0.695} &   {0.172}  &  \textbf{0.037}   & \textbf{0.004 }  & \textbf{0.227}\\
\hline
\end{tabular}
\label{weather}
\end{table}
\begin{table}[htb]
\caption{MAEs of different methods on the weather dataset with various values $N$. Best results are highlighted in bold fonts.}
\label{performance}
\centering
\begin{tabular}{|c|c|c|c|c|c|c|c|}
\hline
\multirow{2}* {Method} & \multicolumn{6}{|c|}{N} \\
\cline{2-7}
 & 1      &  2    &  4     &  6    & 8  & 10   \\
\hline
{PA-I}  &     0.285  &  0.318  &  0.364   & 0.400   & 0.430 & 0.461  \\
\hline
{PA-II} &    0.282 &   0.315   & 0.361  &  0.397  &  0.426  &   0.454 \\
\hline
{iS-PLS} &     0.529 &   0.620  &  0.740 &   0.828  &  0.901 & 0.959\\
\hline
{SOMOR} &      0.283 &    0.316  &   0.362 & 0.398  &  0.426  &   0.454\\
\hline
OTML    &     0.284    & 0.317 &  0.364     &   0.400     &  0.428      &  0.455       \\
\hline
ELLA &    {0.246}  & {0.306} &   {0.379 } &  {0.432}  &  {0.472}  & {0.510} \\
\hline
{MORES} &    \textbf{0.227}  & \textbf{0.239} &   \textbf{0.255 }  & \textbf{0.270}  &  \textbf{0.283}   &\textbf{0.295}\\
\hline
\end{tabular}
\label{freq}
\end{table}

We further test all the methods with different model update frequencies. The experimental setting is as follows: The model is updated when accumulatively receiving N (=1, \ldots, 10) training data points, while the test is still performed on all the data points. As reported in Table \ref{freq}, the prediction accuracies of all the methods are gradually reduced  as $N$ increases. Moreover, for various values of $N$, MORES performs better than the other methods because of dynamic structure learning and the utilization of the historical data. In addition, we can see that the performance of MORES with $N=10$ is comparable to that of ELLA with $N=2$.

\subsection{Sensitivity Analysis}
We also study the sensitivity of parameters in our algorithm on the largest dataset, the weather dataset.
As shown in Fig. \ref{parameter}, with the fixed $\mu$, our method is not sensitive to the parameters with wide ranges, especially for $\beta$ and $\eta$. Therefore, we fix $\beta=1$ and $\eta=100$ throughout the experiments, so as to reduce the cost of tuning parameters.

\begin{figure}
\centering
\subfigure[Fix others, and vary $\alpha$ and $\mu$]{\includegraphics[width=0.49\linewidth]{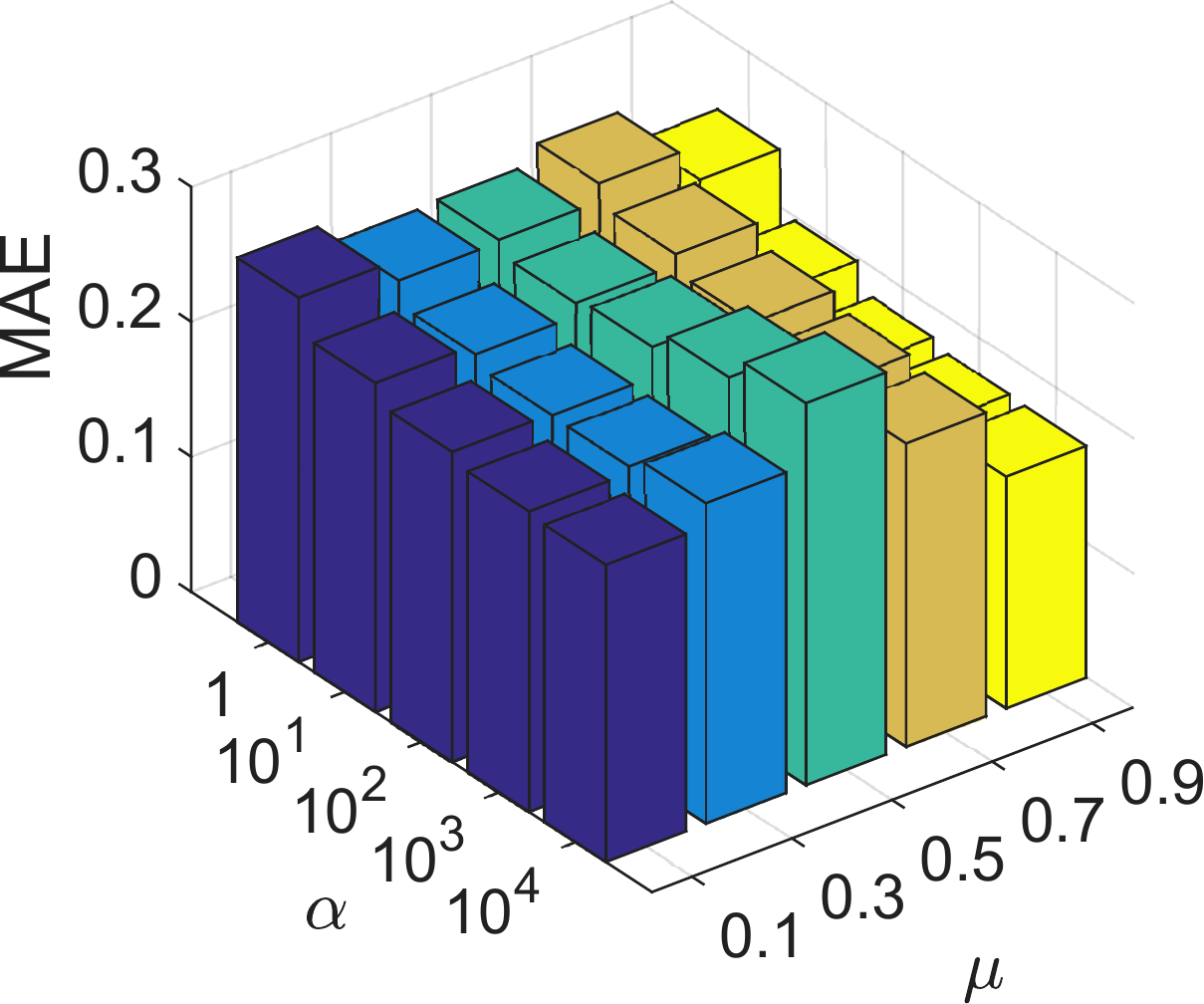}}
\subfigure[Fix others, and vary $\beta$ and $\mu$]{\includegraphics[width=0.49\linewidth]{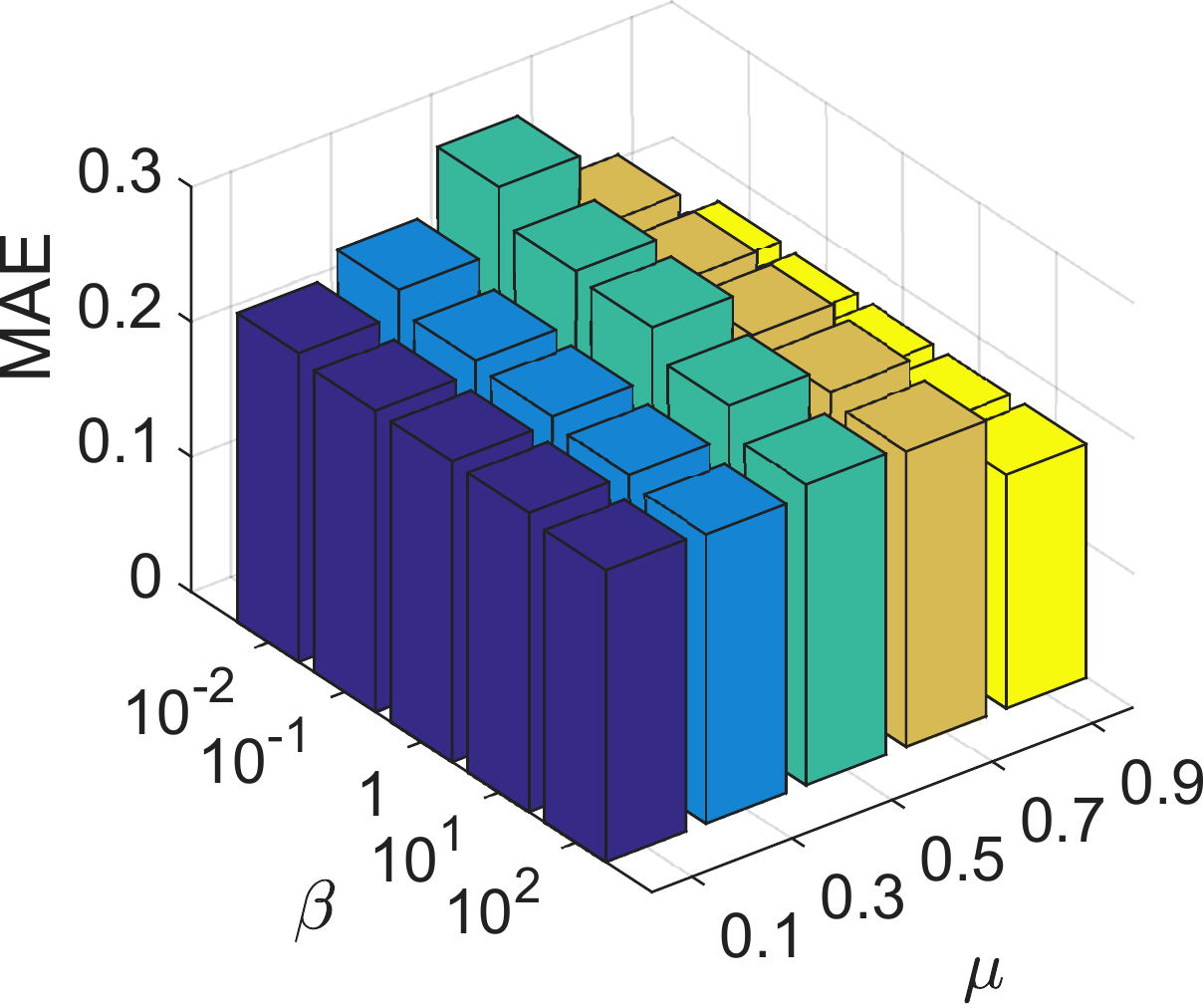}}
\subfigure[Fix others, and vary $\rho$ and $\mu$]{\includegraphics[width=0.49\linewidth]{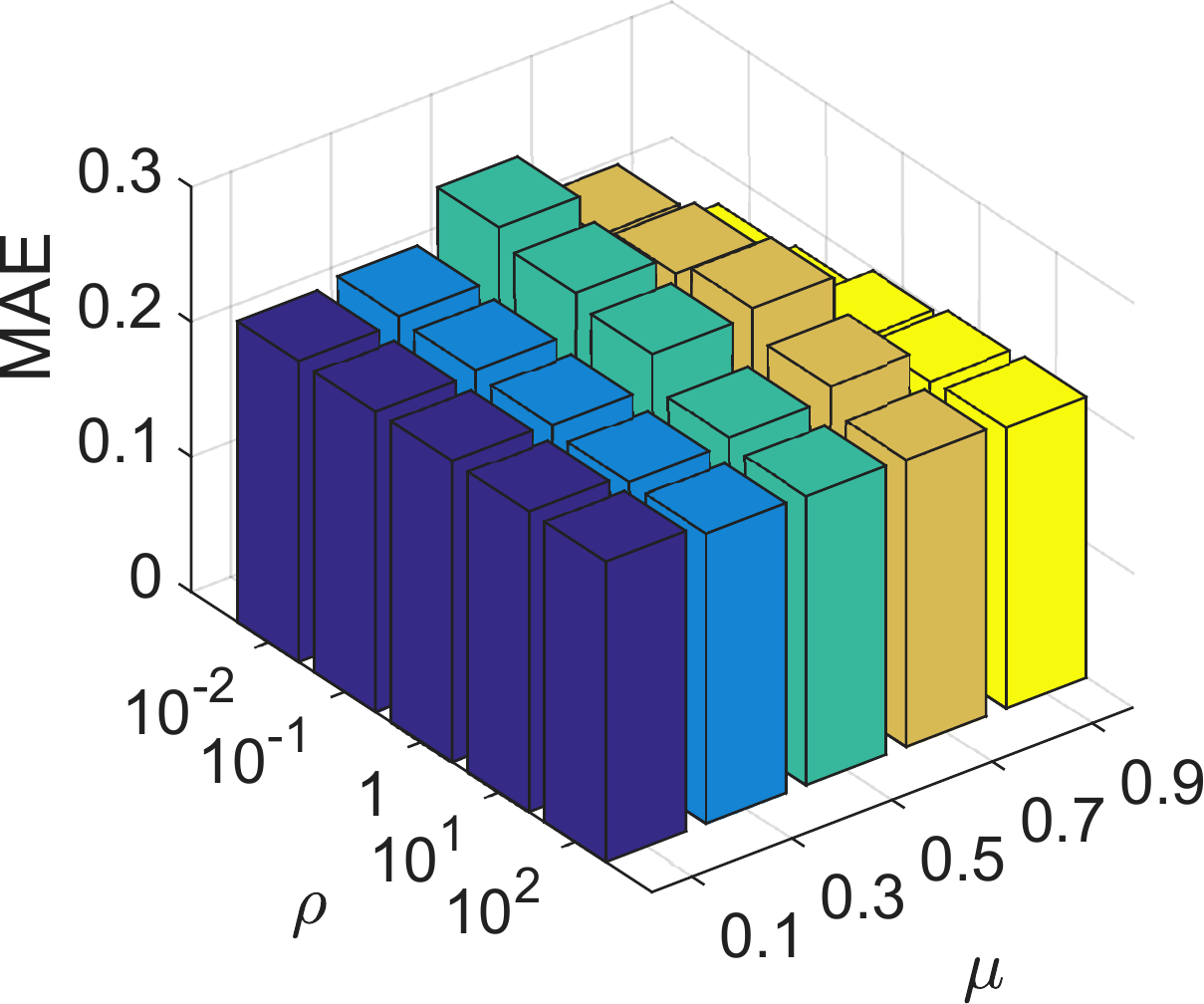}}
\subfigure[Fix others, and vary $\eta$ and $\mu$]{\includegraphics[width=0.49\linewidth]{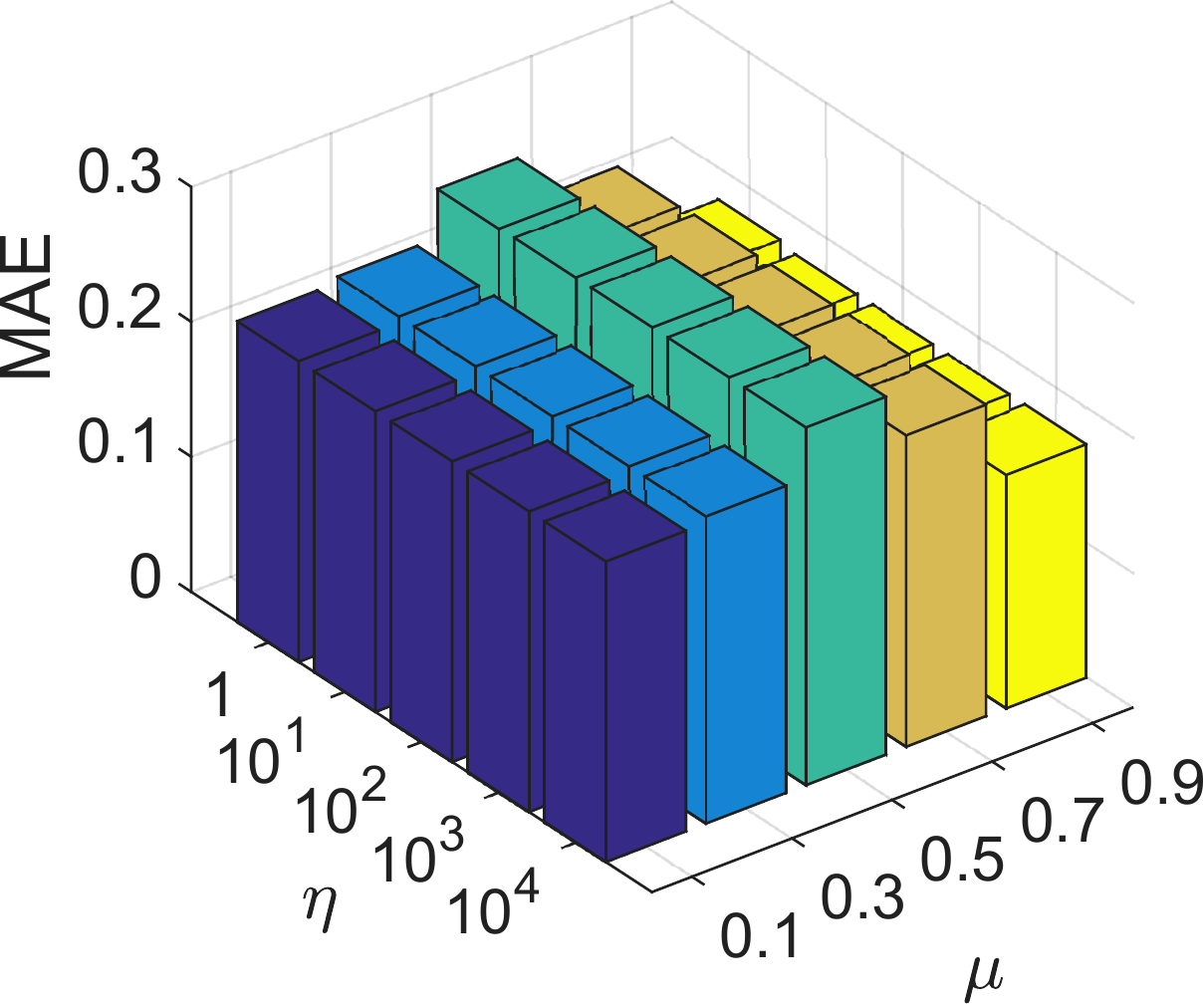}}
\caption{The study of parameter sensitivity on the weather dataset.}
\label{parameter}
\end{figure}

\subsection{Efficiency}
We test the update speeds of our algorithm on the three real-world datasets above. The experiments are conducted on a desktop with Inter(R) Core(TM) i7-4780 CPU, and MORES are implemented using MATLAB R2013b 64bit edition without parallel operation. We compare our method with ELLA, because of its good prediction accuracy based on the experiments above. Fig. \ref{speed} shows the update speeds of the two approaches. ELLA achieves 252, 136, 314 updates per second on the stock, the Barrett WAM, and the weather dataset, respectively, while MORES performs 4352, 2204, 5015 updates per second on the three datasets, repspectively. Based on these figures, MORES is more than 15 times faster than ELLA. In addition, if we apply some parallel implementations or use more efficient programming language, the update speed of MORES can be further improved.
\begin{figure}
\centering
\includegraphics[width=0.85\linewidth]{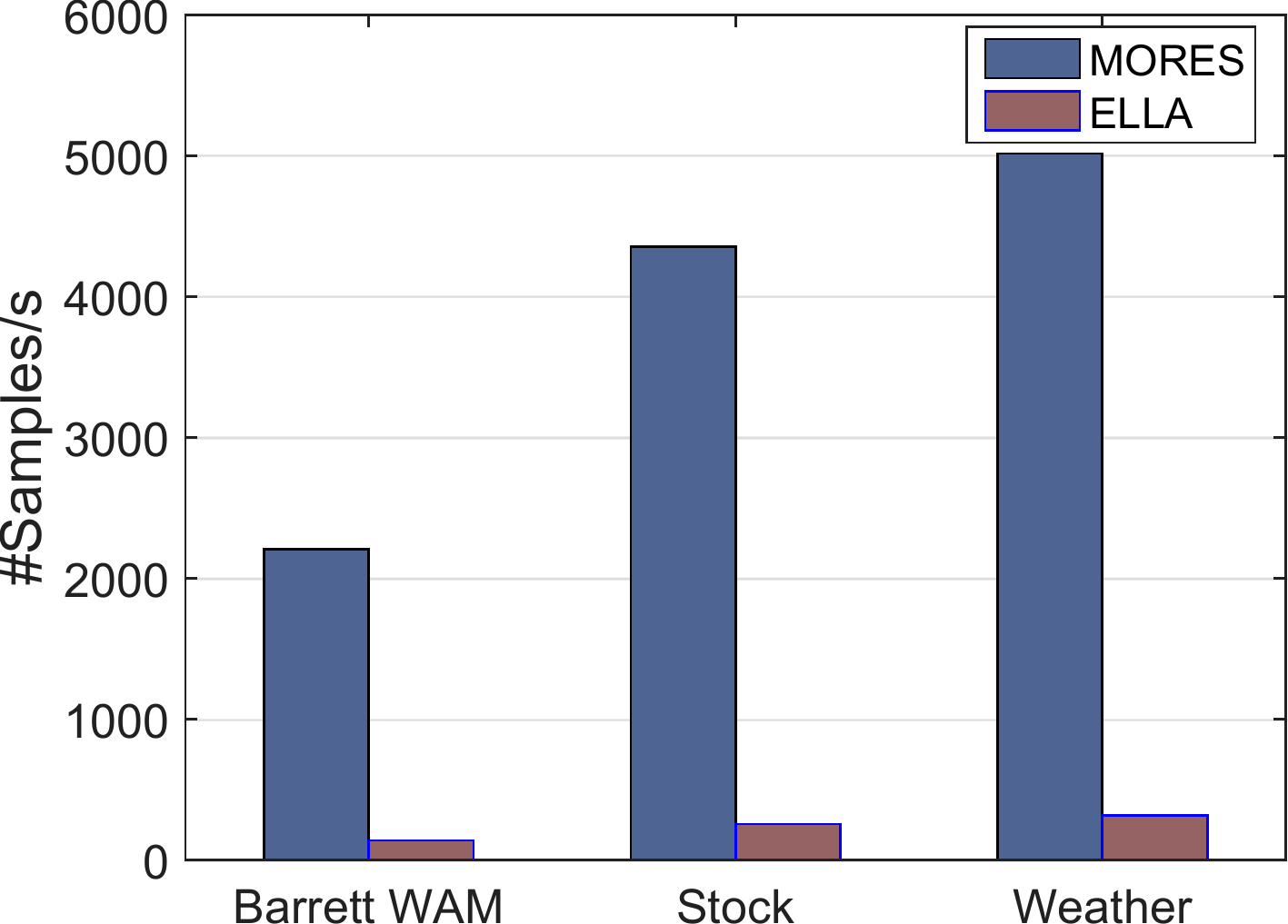}
\caption{Update Speeds of MORES (Unit: $\#$ samples processed per second).}
\label{speed}
\end{figure}

\section{Related Work}

In this section, we review the related works from three aspects: online single-output regression, online multi-task learning and batch multiple-output regression.

\textbf{{Online single-output regression}}: \cite{ma2003accurate} presented an online version of support vector regression algorithm, called (AOSVR).  AOSVR classified all training samples into three distinct auxiliary sets according to the KKT conditions that define the optimal solution. After that, the rules for recursively updating the optimal solution were devised based on this classification. \cite{crammer2006online} proposed a margin based online regression algorithm, called passive-aggressive (PA). PA incrementally updated the model by formalizing the trade-off between the amount of progress made on each round and the amount of information retained from previous rounds. \cite{montana2008learning} proposed an incremental support vector regression algorithm, which evolved a pool of online SVR experts and learned to trade by dynamically weighting the experts' opinions.

\textbf{{Online multi-task learning}}:
Dekel et al. \cite{dekel2006online} aim to online capture the relationships between the tasks, and leverage the relationships to improve the prediction accuracy. They measure each individual prediction with its own individual loss, and then take advantage of a global loss function to evaluate the quality of the multiple predictions made on each step.
 Cavallanti et al. \cite{cavallanti2010linear} proposed an online learning algorithm for multi-task classification. When one training instance of certain task became available, the authors updated the weight vectors for all the tasks simultaneously depending on the rules derived from a pre-defined interaction matrix. The fixed interaction matrix encodes the task relatedness, which is beneficial to multitask classification. Instead of treating the interaction matrix as a priori knowledge, Saha et al. \cite{saha2011online} developed an online learning framework which simultaneously learned the weight vectors for all the classification tasks as well as the interaction matrix adaptively from the data. Moreover, the authors exploited the learned interaction matrix to conduct an active learning extension in an online multi-task learning setting.
 Ruvolo and Eaton \cite{eaton2013ella} propose a method, called ELLA, which aims to maintain a basis to transfers knowledge for learning each new task, and dynamically updates the basis to improve previously learned models.
In order to lower the complexity over \cite{eaton2013ella}, Ruvolo and Eaton \cite{ruvolo2014online} further take advantage of K-SVD to online learn multiple tasks.

\textbf{Batch multiple-output regression}: Many batch multiple-out-put regression algorithms have been proposed, which tried to mine the structure among outputs. Rothman et al. \cite{rothman2010sparse} presented MRCE, which jointly learned the output structure in the form of the noise covariance matrix and the regression coefficients for predicting each output. Sohn and Kim \cite{sohn2012joint} designed an algorithm to simultaneously estimate the regression coefficient vector for each output along with the covariance structure of the outputs with a shared sparsity assumption on the regression coefficient vectors. Rai et al. \cite{rai2012simultaneously} proposed an approach that leveraged the covariance structure of the regression coefficient matrix and the conditional covariance structure of the outputs for learning the model parameters. \cite{kim2012tree} proposed a tree-guided group lasso, or tree lasso, that directly combined statistical strength across multiple related outputs. They estimated the structured sparsity under multi-output regression by employing a novel penalty function constructed from the tree. Since these methods are trained in the batch mode, they are not suitable for online multiple-output prediction.

\section{Conclusions}
In this paper, we proposed a novel {online} multiple-output regression method for streaming data. The proposed method can simultaneously and dynamically learn the structures of both the regression coefficients change and the residual errors, and leverage the learned structure information to continuously update the model. Meanwhile, we accumulated the prediction error on all the seen samples in an incremental way without information loss, and introduced a forgetting factor to weight the samples so as to fit in data streams' evolvement. The experiments were conducted on one synthetic dataset and three real-world datasets, and the experimental results demonstrated the effectiveness and efficiency of the proposed method.

\ifCLASSOPTIONcaptionsoff
  \newpage
\fi

\bibliographystyle{IEEEtran}
\bibliography{bare_jrnl_compsoc}

\end{document}